\documentclass[11pt,epdf]{article}
\usepackage{mathrsfs}
\usepackage{amsmath}
\usepackage{amsfonts}
\usepackage[colorlinks,
        linkcolor=red,
        anchorcolor=blue,
        citecolor=green
        ]{hyperref}
\usepackage{graphicx}
\usepackage{appendix}
\usepackage{float}
\usepackage{latexsym,amsfonts,amsbsy,amssymb}
\usepackage{amsmath,amsthm}
\usepackage{algorithm}
\usepackage{algorithmicx}
\usepackage{algpseudocode}
\usepackage{diagbox}
\usepackage{lipsum}
\usepackage{epsfig,subfigure}
\usepackage{booktabs,longtable}
\usepackage{epstopdf}
\usepackage{multirow}
\usepackage{color}
\usepackage{xcolor}
\usepackage{indentfirst}
\usepackage{enumerate}
\usepackage[font={footnotesize}]{caption}
\makeatletter
\renewcommand{\@thesubfigure}{\hskip\subfiglabelskip}
\makeatother
\allowdisplaybreaks
\newfont{\bb}{msbm10}

\theoremstyle{remark}
\newtheorem{remark}{Remark}[section]

\newtheorem{theorem}{\textbf{Theorem}}[section]
\newtheorem{definition}{Definition}[section]
\newtheorem{lemma}{Lemma}[section]

\renewcommand{\arraystretch}{1.5}
\begin{document}
\date{}
\author{Hongbing~Zhang\thanks{Corresponding author. E-mail:~zhb123abc@163.com.}~, Bing Zheng
\\
\small \text{School of Mathematics and Statistics,  Lanzhou University, Lanzhou 730000, P.R. China}}
\title{A novel non-convex minimax $p$-th order concave penalty function approach to low-rank tensor completion}
\maketitle
\begin{abstract}
The low-rank tensor completion (LRTC) problem aims to reconstruct a tensor from partial sample information, which has attracted significant interest in a wide range of practical applications such as image processing and computer vision. Among the various techniques employed for the LRTC problem, non-convex relaxation methods have been widely studied for their effectiveness in handling tensor singular values, which are crucial for accurate tensor recovery. While the minimax concave penalty (MCP) non-convex relaxation method has achieved promising results in tackling the LRTC problem and gained widely adopted, it exhibits a notable limitation: insufficient penalty on small singular values during the singular value handling process, resulting in inefficient tensor recovery. To address this issue and enhance recovery performance, a novel minimax $p$-th order concave penalty (MPCP) function is proposed. Based on this novel function, a tensor $p$-th order $\tau$ norm is proposed as a non-convex relaxation for tensor rank approximation, thereby establishing an MPCP-based LRTC model. Furthermore, theoretical convergence guarantees are rigorously established for the proposed method. Extensive numerical experiments conducted on multiple real datasets demonstrate that the proposed method outperforms the state-of-the-art methods in both visual quality and quantitative metrics.

\vskip10pt

\end{abstract}

\section{Introduction}
Tensors, as higher-order generalizations of vectors and matrices, are widely acknowledged to capture intrinsic structural information in multimodal/multi-relational data more effectively. The processing of tensor data has been increasingly recognized as a pivotal component across diverse domains, such as image/video processing, magnetic resonance imaging data recovery \cite{wu2025paladin}, machine learning\cite{KONG2023109545,5032019109}, computer vision\cite{LIAO2023109624,YANG2022108311}, and pattern recognition\cite{jiang2023robust,10145070}. Nevertheless, in practical application scenarios, real-world tensor data are frequently corrupted due to unpredictable or unavoidable disturbances, thereby driving the demand for tensor completion techniques to restore multidimensional data integrity.

It is well-established that tensor rank minimization has been widely adopted as a prevalent approach for addressing low-rank tensor completion (LRTC) problem. However, the definition of tensor rank is not unique. The mainstream definitions of tensor rank are CANDECOMP/PARAFAC (CP) rank based on CP decomposition \cite{harshman1970foundations}, Tucker rank based on Tucker decomposition \cite{tucker1966some}, and tubal rank \cite{doi:10.1137/110837711} induced by tensor singular value decomposition (t-SVD) \cite{6416568}. Nevertheless, directly solving the CP rank is NP-hard \cite{139201360}. The tubal rank can better maintain the data structure compared with CP rank and Tucker rank. Subsequently, Zhang et al. \cite{6909886} defined the tensor nuclear norm (TNN) based on tensor tubal rank to solve the LRTC problem and obtained the advanced tensor recovery results. Recently, Zheng et al. \cite{2020170} proposed a new form of rank (N-tubal rank) based on tubal rank, which adopts a new unfold method of higher-order tensors into third-order tensors in various directions. This method enables t-SVD to be applied to higher-order tensor.  

Although the TNN method has become a commonly used method for solving the LRTC problem, it applies the same penalty to all singular values, which will leads to suboptimal solutions. In contrast, non-convex relaxation methods have garnered widespread attention due to their ability to handle tensor singular values more effectively. To achieve non-convex relaxation, various functions and methods have been employed by researchers, including the logarithmic function \cite{gong2013general}, the minimax concave penalty (MCP) function \cite{zhang2010nearly, 2132021245}, and the t-Schatten-p norm \cite{kong2018t}. These methods, with their distinct non-convex penalty mechanisms, offer greater flexibility in handling tensor singular values, leading to improved performance in tensor recovery problem. 

Recently, the MCP function as a non-convex relaxation has achieved promising results in the LRTC problem \cite{chen2024spatiotemporal,2132021245, qiu2021nonlocal,you2019nonconvex,ZHANG2024110253}. In \cite{2132021245}, the MCP function with different parameters were employed to penalize different singular values, enhancing overall effectiveness. However, as the problem scale increases, parameter selection for each singular value becomes extremely challenging. The bivariate equivalent MCP method proposed in \cite{ZHANG2024110253} enables parameters to be updated iteratively, significantly improving recovery performance. \cite{chen2024spatiotemporal} proposed a truncated MCP method to penalize singular values, achieving enhanced results by adjusting truncated rate and corresponding MCP function parameters. Essentially, these methods still rely on various parameter adjustment strategies of the MCP function to obtain good recovery results. Nevertheless, the inherent limitation of the MCP function, inadequate penalization of small singular values, remains unresolved by the above methods. To address this issue, a novel minimax $p$-th order concave penalty (MPCP) function is proposed. The proposed MPCP function not only preserves the MCP function’s property of protecting large singular values but also imposes stronger penalties on small singular values. Specifically, the main contributions of this paper are summarized as follows. 


First, a novel MPCP function is proposed to address the deficiency of the MCP function in insufficient penalization of small singular values. Key properties of the MPCP function are systematically analyzed and mathematically proven. To enable efficient optimization for LRTC applications, the proximal operator for the MPCP function is derived. Furthermore, tensor extensions of the MPCP function, including tensor $p$-th order $\tau$ norm definition and associated mathematical properties, are thoroughly investigated.

Second, a new MPCP-based LRTC model is developed, accompanied by a solving algorithm implemented through the alternating direction multipliers method (ADMM). Rigorous theoretical analysis of the algorithm's convergence is provided to guarantee the proposed method numerical stability.

Third, comprehensive experiments are conducted on multi-dimensional datasets, including third-order tensor, fourth-order tensor, and fifth-order tensor, to evaluate the proposed method's generalization capabilities. Experimental results demonstrate that the proposed MPCP method significantly outperforms the state-of-the-art methods. A dedicated comparative analysis against MCP-based variants further reveals the MPCP method's superior performance, thereby corroborating our theoretical analysis.

The rest of the paper is organized as follows. In Section 2, some preliminary knowledge and background of the tensors are given. Definitions and theorems about the MPCP function are presented in Section 3. The main results, including the proposed model and algorithm, are shown in Section 4. In Section 5, we give the convergency analysis for the proposed method. The results of extensive experiments and discussion are presented in Section 6. Conclusions are drawn in Section 7.
\section{Preliminaries}
In this section, we give some basic notations and briefly introduce some definitions used throughout the paper. Generally, a lowercase letter and an uppercase letter denote a vector $b$ and a matrix $B$, respectively. A calligraphic upper case letter $\mathcal{B}\in \mathbb{R}^{\mathit{J}_{1}\times\mathit{J}_{2}\times\cdots\times\mathit{J}_{N}}$ denotes an $N$th-order tensor and $\mathcal{B}_{j_{1},j_{2},\cdots,j_{N}}$ is its $(j_{1},j_{2},\cdots,j_{N})$-th element. The Frobenius norm of a tensor is defined as $\|\mathcal{B}\|_{F}=(\sum_{j_{1},j_{2},\cdots,j_{N}}\mathcal{B}_{j_{1},j_{2},\cdots,j_{N}}^{2})^{1/2}$.
For a third-order tensor $\mathcal{B}\in\mathbb{R}^{\mathit{J}_{1}\times\mathit{J}_{2}\times\mathit{J}_{3}}$,  we use the Matlab notation $\mathcal{B}(k,:,:)$, $\mathcal{B}(:,i,:)$ and $\mathcal{B}(:,:,j):=\mathcal{B}^{(j)}$ to denote its $k$th horizontal slice, $i$th lateral slice and $j$th frontal slice, respectively, and $\mathcal{B}(k,i,:)$ is denoted $(k,i)$-th tube of $\mathcal{B}$.
\begin{definition}[Tensor mode-$n$ unfolding and folding \cite{12345152009}]
	The mode-$n$ unfolding of an $N$th-order tensor $\mathcal{B}\in \mathbb{R}^{\mathit{J}_{1}\times\mathit{J}_{2}\times\cdots\times\mathit{J}_{N}}$ is denoted as a matrix $\mathcal{B}_{(n)}\in\mathbb{R}^{\mathit{J}_{n}\times\mathit{J}_{1}\cdots\mathit{J}_{n-1}\mathit{J}_{n+1}\cdots\mathit{J}_{N}} $. Tensor element $(j_{1}, j_{2},...,j_{N} )$ maps to matrix element $(j_{n}, l)$, where
	\begin{equation*}
		l=1+\sum_{i=1,i\neq n}^{N}(j_{i}-1)\mathit{L}_{i}\quad \text{with}\quad \mathit{L}_{i}=\prod_{k=1,k\neq n}^{i-1}\mathit{J}_{k}. 
	\end{equation*}
	The mode-$n$ unfolding operator and its inverse are respectively denoted as ${\tt unfold}_{n}$ and ${\tt fold}_{n}$, and they satisfy $\mathcal{B}={\tt fold}_{n}(\mathcal{B}_{(n)})={\tt fold}_{n}({\tt unfold}_{n}(\mathcal{B}))$.
\end{definition}
\begin{definition}[The mode-$n$ product of tensor \cite{12345152009}]
	The mode-$n$ product of an $N$th-order tensor  $\mathcal{B}\in \mathbb{R}^{\mathit{J}_{1}\times\mathit{J}_{2}\times\cdots\times\mathit{J}_{N}}$ with matrix $M\in\mathbb{R}^{\mathit{I}_{n}\times \mathit{J}_{n}}$ is denoted by $\mathcal{C}=\mathcal{B}\times_{n}M$. Elementwise, we have
	\begin{equation*}
		\mathcal{C}=\mathcal{B}\times_{n}M\in\mathbb{R}^{\mathit{J}_{1}\times\mathit{J}_{2}\times\cdots\mathit{J}_{n-1}\times\mathit{I}_{n}\times\mathit{J}_{n+1}\cdots\mathit{J}_{N}}\quad\Leftrightarrow\quad \mathcal{C}_{(n)}=M  \mathcal{B}_{(n)}. 
	\end{equation*} 
\end{definition}
Now we review the Discrete Fourier Transformation (DFT) for tensor-tensor product. For a third-order tensor $\mathcal{B}\in\mathbb{C}^{\mathit{J}_{1}\times\mathit{J}_{2}\times\mathit{J}_{3}}$, let $\bar{\mathcal{B}}\in\mathbb{C}^{\mathit{J}_{1}\times\mathit{J}_{2}\times\mathit{J}_{3}}$ be the result of DFT of $\mathcal{B}$ along the third mode. Specifically, let $F=[\textbf{f}_{1},\dots,\textbf{f}_{\mathit{J}_{3}}]\in\mathbb{R}^{\mathit{J}_{3}\times\mathit{J}_{3}}$, where $	\textbf{f}_{k}=[\omega^{0\times(k-1)};\dots;\omega^{(\mathit{J}_{3}-1)\times(k-1)}]\in\mathbb{R}^{\mathit{J}_{3}}$, $ \omega=e^{-\frac{2\pi t}{\mathit{J}_{3}}}$ and $t=\sqrt{-1}$. Then $\bar{\mathcal{B}}(k,i,:)=F\mathcal{B}(k,i,:),$ which can be computed by Matlab command $\bar{\mathcal{B}}=\text{\tt fft}(\mathcal{B},[~],3)=\mathcal{B}\times_{3}F$. Furthermore, the inverse DFT is computed by command $\text{\tt ifft}$ satisfying $\mathcal{B}=\text{\tt ifft}(\bar{\mathcal{B}},[~],3)$. For a third-order tensor $\mathcal{B}\in\mathbb{R}^{\mathit{J}_{1}\times\mathit{J}_{2}\times\mathit{J}_{3}}$, the block circulation operation is defined as
\begin{equation*}
	\text{\tt bcirc}(\mathcal{B}):=
	\begin{pmatrix}
		\mathcal{B}^{(1)}& \mathcal{B}^{(\mathit{J}_{3})}&\dots& \mathcal{B}^{(2)}&\\
		\mathcal{B}^{(2)}& \mathcal{B}^{(1)}&\dots& \mathcal{B}^{(3)}&\\
		\vdots&\vdots&\ddots&\vdots&\\
		\mathcal{B}^{(\mathit{J}_{3})}& \mathcal{B}^{(\mathit{J}_{3}-1)}&\dots& \mathcal{B}^{(1)}&
	\end{pmatrix}\in\mathbb{R}^{\mathit{J}_{1}\mathit{J}_{3}\times\mathit{J}_{2}\mathit{J}_{3}}.
\end{equation*}
The block diagonalization operation and its inverse operation are respectively defined as $\text{\tt bdiag}(\mathcal{B}):=\text{\tt diag}(\mathcal{B}^{(1)},\mathcal{B}^{(2)},\cdots,\mathcal{B}^{(\mathit{J}_{3})})\in\mathbb{R}^{\mathit{J}_{1}\mathit{J}_{3}\times\mathit{J}_{2}\mathit{J}_{3}}$ and $\text{\tt bdfold}(\text{\tt bdiag}(\mathcal{B})):=\mathcal{B}$. The block vectorization operation and its inverse operation are respectively defined as 
\begin{equation*}
	\text{\tt bvec}(\mathcal{B}):=\begin{pmatrix}
		\mathcal{B}^{(1)}\\\mathcal{B}^{(2)}\\\vdots\\\mathcal{B}^{(\mathit{J}_{3})}
	\end{pmatrix}\in\mathbb{R}^{\mathit{J}_{1}\mathit{J}_{3}\times\mathit{J}_{2}}~\text{and}~\text{\tt bvfold}(\text{\tt bvec}(\mathcal{B})):=\mathcal{B}.
\end{equation*} 

\begin{definition}[T-product \cite{6416568}]
	Let $\mathcal{B}\in\mathbb{R}^{\mathit{J}_{1}\times\mathit{J}_{2}\times\mathit{J}_{3}}$ and $\mathcal{C}\in\mathbb{R}^{\mathit{J}_{2}\times\mathit{I}\times\mathit{J}_{3}}$. The t-product $\mathcal{B}\ast\mathcal{C}$ is defined to be a tensor of size $\mathit{J}_{1}\times\mathit{I}\times\mathit{J}_{3}$,
	\begin{equation*}
		\mathcal{B}\ast\mathcal{C}:=\text{\tt bvfold}(\text{\tt bcirc}(\mathcal{B})\text{\tt bvec}(\mathcal{C})).
	\end{equation*}
	Since that the circular convolution in the spatial domain is equivalent to the element-wise multiplication in the Fourier domain, the t-product between two tensors $\mathcal{C}=\mathcal{A}\ast\mathcal{B}$ is equivalent to
	$\bar{\mathcal{C}}=\text{\tt bdfold}(\text{\tt bdiag}(\bar{\mathcal{A}})\text{\tt bdiag}(\bar{\mathcal{B}})).$
\end{definition} 
The conjugate transpose of a tensor $\mathcal{B}\in\mathbb{C}^{\mathit{J}_{1}\times\mathit{J}_{2}\times\mathit{J}_{3}}$ is the tensor $\mathcal{B}^{H}\in\mathbb{C}^{\mathit{J}_{2}\times\mathit{J}_{1}\times\mathit{J}_{3}}$ \cite{6416568} obtained by conjugate transposing each of the frontal slices and then reversing the order of transposed frontal slices 2 through $\mathit{J}_{3}$. 
Besides, the identity tensor $\mathcal{I}\in\mathbb{R}^{\mathit{J}_{1}\times\mathit{J}_{1}\times\mathit{J}_{3}}$ is the tensor whose first frontal slice is the $\mathit{J}_{1}\times\mathit{J}_{1}$ identity matrix, and whose other frontal slices are all zeros. It is easy to get $\mathcal{B}\ast\mathcal{I}=\mathcal{B}=\mathcal{I}\ast\mathcal{B}$. A tensor $\mathcal{Q}\in\mathbb{R}^{\mathit{J}_{1}\times\mathit{J}_{1}\times\mathit{J}_{3}}$ is orthogonal if it satisfies $\mathcal{Q}\ast\mathcal{Q}^{H}=\mathcal{Q}^{H}\ast\mathcal{Q}=\mathcal{I}.$ A third-order tensor is called f-diagonal if each of its frontal slices is a diagonal matrix.
\begin{lemma}[T-SVD \cite{8606166}]
	Let $\mathcal{B}\in\mathbb{R}^{\mathit{J}_{1}\times\mathit{J}_{2}\times\mathit{J}_{3}}$ be a third-order tensor, then it can be factored as 
	\begin{equation*}
		\mathcal{B}=\mathcal{U}\ast\mathcal{S}\ast\mathcal{V}^{H},
	\end{equation*}
	where $\mathcal{U}\in\mathbb{R}^{\mathit{J}_{1}\times\mathit{J}_{1}\times\mathit{J}_{3}}$ and $\mathcal{V}\in\mathbb{R}^{\mathit{J}_{2}\times\mathit{J}_{2}\times\mathit{J}_{3}}$ are orthogonal tensors, and $\mathcal{S}\in\mathbb{R}^{\mathit{J}_{1}\times\mathit{J}_{2}\times\mathit{J}_{3}}$ is an f-diagonal tensor.
\end{lemma}

\begin{definition}[Tensor tubal rank \cite{6909886}]
	The tubal rank of a tensor $\mathcal{B}\in\mathbb{R}^{\mathit{J}_{1}\times\mathit{J}_{2}\times\mathit{J}_{3}}$, denoted as ${\tt rank}_{t}(\mathcal{B})$, is defined to be the number of non-zero singular tubes of $\mathcal{S}$, where $\mathcal{S}$ comes from the t-SVD of $\mathcal{B}:\mathcal{B}=\mathcal{U}\ast\mathcal{S}\ast\mathcal{V}^{H}$. That is ${\tt rank}_{t}(\mathcal{B})=\#\{j:\mathcal{S}(j,j,:)\neq0\}.$
\end{definition}
\begin{definition}[Tensor nuclear norm (TNN)]
	The tensor nuclear norm of a tensor $\mathcal{B}\in\mathbb{R}^{\mathit{J}_{1}\times\mathit{J}_{2}\times\mathit{J}_{3}}$, denoted as $\|\mathcal{B}\|_{TNN}$, is defined as the sum of the singular values of all the frontal slices of $\bar{\mathcal{B}}$, i.e.,
	\begin{equation*}
		\|\mathcal{B}\|_{TNN}:=\frac{1}{\mathit{J}_{3}}\sum_{j=1}^{\mathit{J}_{3}}\|\bar{\mathcal{B}}^{(j)}\|_{\ast}=\frac{1}{\mathit{J}_{3}}\sum_{j=1}^{\mathit{J}_{3}}\sum_{i=1}^{R}\sigma_{i}(\bar{\mathcal{B}}^{(j)}),
	\end{equation*}
	where $\bar{\mathcal{B}}^{(j)}$ is the $j$-th frontal slice of $\bar{\mathcal{B}}$, with $\bar{\mathcal{B}}=\text{\tt fft}(\mathcal{B},[~],3)$; $R=\min(\mathit{J}_{1},\mathit{J}_{2})$; $\sigma_{j}(\bar{\mathcal{B}}^{(j)})$ is the $i$-th singular value of $\bar{\mathcal{B}}^{(j)}$.
\end{definition}

In order to simplify the representation of the correlation between pairs of dimensions of a tensor, a new tensor mode-$q$ unfolding and folding definition has been proposed to replace the tensor mode-$k_1,k_2$ unfolding and folding definition in \cite{2020170}.

\begin{definition}[Tensor mode-$q$ unfolding and folding]
	For an $N$th-order tensor $\mathcal{B}$, its mode-$q$ unfolding is a third-order tensor denoted by $\mathcal{B}_{<q>}\in\mathbb{R}^{\mathit{J}_{q_{1}}\times\mathit{J}_{q_{2}}\times\prod_{s\neq q_{1},q_{2}}\mathit{J}_{s}}$. Mathematically, the  $(j_{1}, j_{2},...,j_{N} )$-th element of $\mathcal{B}$ maps to the $(j_{q_{1}},j_{q_{2}},k)$-th element of $\mathcal{B}_{<q>}$, where
	\begin{alignat*}{2}
		&k=1+\sum_{s=1,s\neq q_{1}, q_{2}}^{N}(j_{s}-1)\prod_{m=1,m\neq q_{1}, q_{2}}^{s-1}\mathit{J}_{m},
		\\&q=(q_1-1)(N-q_1/2)+q_2-q_1,~~1\leqslant q_1<q_2 \leqslant N.	
	\end{alignat*}
	The mode-$q$ unfolding operator and its inverse operation are respectively denoted as $\mathcal{B}_{<q>}:=\text{\tt t-unfold}(\mathcal{B},q)$ and $\mathcal{B}:=\text{\tt t-fold}(\mathcal{B}_{<q>},q)$. Additionally, $q$ ranges from $1$ to $N(N-1)/2$.
\end{definition}
\begin{definition}[N-tubal rank \cite{2020170}]
	The N-tubal rank of an $N$th-order tensor $\mathcal{B}$ is defined as a vector, the elements of which
	contain the tubal rank of all mode-$q$ unfolding tensors, i.e.,
	\begin{equation*}
		N-{\tt rank}_{t}({\mathcal{B}}):=({\tt rank}_{t}({\mathcal{B}}_{<1>}),{\tt rank}_{t}({\mathcal{B}}_{<2>}),\cdots,{\tt rank}_{t}({\mathcal{B}}_{<N(N-1)/2>}))\in\mathbb{R}^{N(N-1)/2}.
	\end{equation*}
\end{definition}
\begin{definition}[MCP function \cite{zhang2010nearly}]
	Let $\tau>1$. The MCP function $h_{\tau}:\mathbb{R}\to\mathbb{R}_{\geqslant0}$ is defined as 
	\begin{equation}
		h_{\tau}(x)=\begin{cases}
			\rvert x\rvert-\frac{x^{2}}{2\tau},\quad& \rvert x\rvert\leqslant\tau,
			\\\frac{\tau}{2},\qquad &\rvert x\rvert>\tau,
		\end{cases} \label{scalarMCP}
	\end{equation}where $\mathbb{R}_{\geqslant0}$ denotes the domain of non-negative real numbers.
\end{definition}
Based on the definition of the MCP function, it can be observed that the function controls the penalty threshold in the LRTC problem by adjusting the $\tau$ value. As $\tau$ increases, the penalty on small singular values gradually intensifies, but the protection range for large singular values narrows. Conversely, decreasing $\tau$ can expand the protection range for large singular values, but it weakens the penalty on small singular values. 
\section{Minimax $p$-th order concave penalty function}
In this section, to address the limitations of the MCP function,  the definition of the MPCP function is proposed.
\begin{definition}[MPCP function]
	Let $\tau>1,$ $0<p < 1$. The MPCP function $\psi_{\tau,p}:\mathbb{R}\to\mathbb{R}_{\geqslant0}$ is defined as 
	\begin{equation}
		\psi_{\tau,p}(x)=\begin{cases}
			\rvert x\rvert-\frac{\rvert x\rvert^{1+p}}{(1+p)\tau}, \quad&\rvert x\rvert\leqslant \tau^{1/p},
			\\\frac{p\tau^{1/p}}{1+p}, &\rvert x\rvert> \tau^{1/p}.
		\end{cases}\label{scalarMPCP}
	\end{equation}
\end{definition}
Compared to the MCP function, the proposed MPCP function introduces an additional $p$ parameter. It is evident that when $p=1$, the MPCP function is identical to the MCP function, while for $p<1$, there is a clear difference between the MPCP function and the MCP function. This indicates that the MPCP function has a broader definition than the MCP function. Additionally, the MPCP function can enhance the penalty on small singular values by decreasing the $p$ value. Fig. \ref{FUCT} shows the function plot for the MCP function with $\tau = 2$, as well as the function plots for the MPCP function with $p$ values ranging from $0.1$ to $0.9$ and $\tau=2^{1/p}$. All plots have the boundary at $x=\pm2$. From the figure, it is clear that as the $p$ value decreases, the penalty on $x$ increases, which demonstrates that the MPCP function can impose a stronger penalty on small singular values compared to the MCP function in the LRTC problem. In the LRTC problem, larger singular values typically represent important information, such as contours and smooth regions, while smaller singular values are primarily composed of noise or outliers \cite{Gu_2014_CVPR}. Therefore, protecting large singular values and enhancing the penalty on smaller singular values are both crucial for ensuring recovery performance and improving image restoration. The MPCP function effectively possesses both of these properties. Additionally, we further analyzed some properties of the MPCP function. Since the MPCP function is symmetric function, it is sufficient to analyze its properties on $[ 0,+\infty )$.

\begin{figure}[!tbh]
	\centering
	\includegraphics[width=0.8\linewidth]{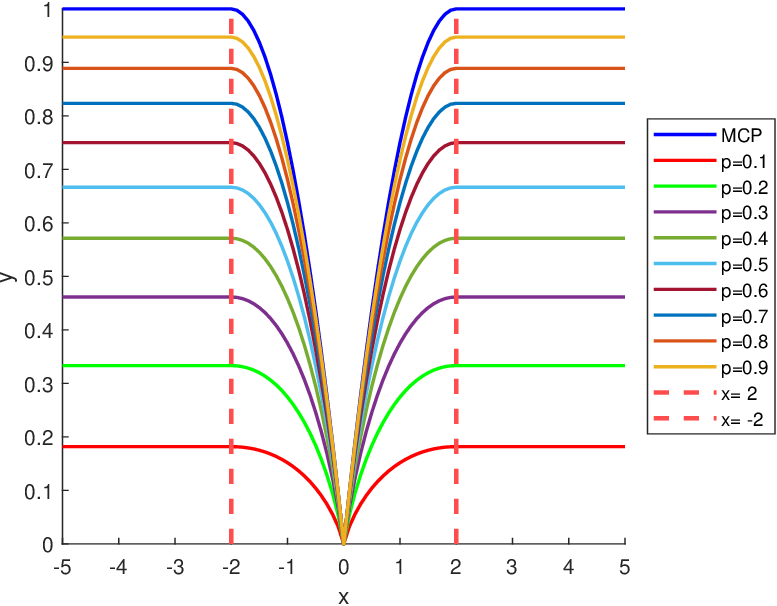}
	\caption {Plots for $ x=2$, $ x=-2$, the MCP and MPCP functions. Here $\tau=2$ for the MCP function; $p=\{0.1,0.2,\cdots,0.9\}$, $\tau=2^{1/p}$ for the MPCP function. }
	\label{FUCT}
\end{figure}

\begin{theorem}\label{propert}
	The MPCP function defined in \eqref{scalarMPCP} satisfies the following properties:
	\\\textbf{1.} $\psi_{\tau,p}(x)$  is continuous, smooth on $[ 0,+\infty ) $ and 
	$
	\psi_{\tau,p}(0)=0,\lim\limits_{x\to+\infty}\frac{\psi_{\tau,p}(x)}{x}=0;$
	\\\textbf{2.} $\psi_{\tau,p}(x)$ is monotonically non-decreasing and concave on $[ 0,+\infty ) $;
	\\\textbf{3.} $\psi_{\tau,p}(x)$ is subadditive;
	\\\textbf{4.} $\psi'_{\tau,p}(x)$ is non-negativity and monotonicity non-increasing on $ [ 0,+\infty) $. Moreover, it is Lipschitz bounded, i.e., there exists constant $L(\ell)$ such that
	\begin{equation*}
		\rvert \psi'_{\tau,p}(x)-\psi'_{\tau,p}(y)\rvert\leqslant L(\ell)\rvert x-y\rvert;
	\end{equation*}
	\textbf{5.} Especially, for the MPCP function, it is increasing as the parameter $\tau$ increases, and
	\begin{equation*}
		\lim\limits_{\tau\to+\infty}\psi_{\tau,p}(x)=\rvert x\rvert,
	\end{equation*}
	which indicates that as the parameter $\tau$ increases, the MPCP function becomes closer to $\rvert x\rvert$. 
\end{theorem}

\begin{proof}
	\textbf{1.}  $\lim_{-}\psi_{\tau,p}(\tau^{1/p})=\lim_{+}\psi_{\tau,p}(\tau^{1/p})=p\tau^{1/p}/(1+p)$ and $\psi'_{-\tau,p}(\tau^{1/p})=\psi'_{+\tau,p}(\tau^{1/p})=0$, thus it's continuous and smooth;
	At last, the conclusions $\psi_{\tau,p}(0)=0$ and $\lim\limits_{x\to+\infty}\frac{\psi_{\tau,p}(x)}{x}=0$ are easily to verified though the formulas in \eqref{scalarMPCP}.
	\\\textbf{2.} This conclusion is direct from its first order and second order derivative function. Its first order and second order derivative functions are as follows:
	\begin{alignat}{2}
		&\psi'_{\tau,p}(x)=\begin{cases}
			1-\frac{\rvert x\rvert^p}{\tau},~&\rvert x\rvert\leqslant \tau^{1/p},
			\\0, &\rvert x\rvert> \tau^{1/p}.
		\end{cases} \label{firstorder}
		\\&\psi''_{\tau,p}(x)=\begin{cases}
			-\frac{p\rvert x\rvert^{p-1}}{\tau},~&\rvert x\rvert\leqslant \tau^{1/p},
			\\0,&\rvert x\rvert> \tau^{1/p}.
		\end{cases}
	\end{alignat}
	It can be find that its first order derivative function is non-negative and its second order derivative function is non-positive, thus $\psi_{\tau,p}(x)$ is concave and monotonically non-decreasing on $[0,+\infty )$.
	\\\textbf{3.} For $x_1\geqslant0$ and $x_2\geqslant0$, concavity implies
	\begin{equation*}
		\psi_{\tau,p}(x_1)=\psi_{\tau,p}(\frac{x_1}{x_1+x_2}(x_1+x_2)+\frac{x_2}{x_1+x_2}0)\geqslant\frac{x_1}{x_1+x_2}\psi_{\tau,p}(x_1+x_2)+\frac{x_2}{x_1+x_2}\psi_{\tau,p}(0),
	\end{equation*}
	and
	\begin{equation*}
		\psi_{\tau,p}(x_2)=\psi_{\tau,p}(\frac{x_2}{x_1+x_2}(x_1+x_2)+\frac{x_1}{x_1+x_2}0)\geqslant\frac{x_2}{x_1+x_2}\psi_{\tau,p}(x_1+x_2)+\frac{x_1}{x_1+x_2}\psi_{\tau,p}(0).
	\end{equation*}
	Then,
	\begin{equation*}
		\psi_{\tau,p}(x_1)+\psi_{\tau,p}(x_2)\geqslant \psi_{\tau,p}(x_1+x_2)+\psi_{\tau,p}(0)=\psi_{\tau,p}(x_1+x_2).
	\end{equation*}
	\textbf{4.} The non-negativity and monotonicity of $\psi'_{\tau,p}(x)$ is direct from the formulas presented in \eqref{firstorder}. Next, we verify its Lipschitz bounded. The proof is mainly based on that $\psi''_{\tau,p}(x)\leqslant0$ and monotonically non-decreasing on $(0,+\infty)$, which turns
	that $\psi''_{\tau,p}(x)$ is always bounded. Thus exists constant $L(\ell):=\max\{\psi''_{\tau,p}(x),\psi''_{\tau,p}(y)\}$ for any $x,y\in(0,+\infty)$, we have
	\begin{equation*}
		\rvert \psi'_{\tau,p}(x)-\psi'_{\tau,p}(y)\rvert\leqslant L(\ell)\rvert x-y\rvert.
	\end{equation*}
	\textbf{5.} Consider $\psi_{\tau,p}(x)$ is a function with respect to $\tau$ when $x$ and $p$ are fixed, then its derivative function is computed as follows:
	\begin{equation*}
		\begin{cases}
			\frac{\tau^{1/p-1}}{1+p},~&\tau<\rvert x\rvert^p,
			\\\frac{\rvert x\rvert^{1+p}}{(1+p)\tau^{2}}, & \tau\geqslant\rvert x\rvert^p.
		\end{cases}
	\end{equation*}
	It demonstrates that MPCP function is increasing in $\tau$ since its derivative function is non-negative. Note that as $\tau\to+\infty$,
	\begin{equation*}
		\rvert x\rvert-\frac{\rvert x\rvert^{1+p}}{(1+p)\tau}\to \rvert x\rvert.
	\end{equation*}
	Then the limit results follow easily. This completes the proof.
\end{proof} 
To facilitate the application of the MPCP function in the LRTC problem, the proximal operator for the MPCP function is proposed.
\begin{theorem}[Proximal operator for the MPCP function]\label{tMPCP}
	Consider the MPCP function given in \eqref{scalarMPCP}, $\tau>1$, $1>p>0$. Its proximal operator is denoted by $\mathit{P}_{\rho}:\mathbb{R}\to\mathbb{R}$ and defined as follows:
	\begin{equation}
		\mathit{P}_{\rho}(y)=\arg\min_{x}\left\lbrace\frac{1}{2}(x-y)^{2}+\rho \psi_{\tau,p}(x)\right\rbrace ,\label{smpcp1}
	\end{equation}
	which can be represented by 
	\begin{equation} 
		\mathit{P}_{\rho}(y)=x_{1}\odot\text{\tt  sign}(y) ~~\text{with}~~x_{1}=\begin{cases}
			0, ~&\rvert y\rvert< h_a,
			\\x_{a}, &\rvert y\rvert= h_a,
			\\x_{*}, &\tau^{1/p}> \rvert y\rvert> h_a,
			\\\rvert y\rvert,&\rvert y\rvert\geqslant \tau^{1/p},
		\end{cases}\label{smpcp2} 
	\end{equation}
	where $x_{a}=(\frac{2\rho p}{(1+p)\tau})^{\frac{1}{1-p}}$, $h_{a}=x_{a}-\rho x_{a}^p/\tau+\rho$, $\odot$ is hadamard product. For $\rvert y\rvert> h_a,x_{*}\in(x_{a},\rvert y\rvert)$ solves:
	\begin{equation}
		x-\rho x^p/\tau+\rho=\rvert y\rvert ~~\text{where} ~~x>0. \label{pro1}
	\end{equation}
	When $\rvert y\rvert>h_{a}$, there are two solutions to \eqref{pro1} and $x_{\ast}$ is the larger one which can be computed from the iteration:
	\begin{equation}
		x_{k+1}=\alpha(x_{k})~\text{where}~\alpha(x)=\rvert y\rvert-\rho+\rho x^p/\tau,\label{pro2}
	\end{equation}
	with the initial condition $x_{0}\in[x_{a},\rvert y\rvert]$.
\end{theorem}
\begin{proof}
	According to the definition of $\psi_{\tau,p}(x)$, when $ \rvert x\rvert\geqslant \tau^{1/p}$, $\mathit{P}_{\rho}(y) = y$, and $\rvert y \rvert=0$, $\mathit{P}_{\rho}(y) = 0$. Next, we consider the case $ 0<\rvert x\rvert< \tau^{1/p}$. Since $x$ and $y$ have the same sign and $\psi_{\tau,p}(x)$ is a symmetric function, it is sufficient to consider only the case where $y>0$ and $x>0$. Let $g(x)=\frac{1}{2}(x-y)^{2}+\rho( x-\frac{ x^{1+p}}{(1+p)\tau})$. Noting that $g(0)=\frac{y^2}{2}$. As $g(x)$ is differentiable for $\tau^{1/p}\geqslant x>0$, re-arranging $g'(x)=0$ gives:
	\begin{equation}
		m(x):=x+\rho-\rho\frac{x^p}{\tau}=y.\label{mx}
	\end{equation}
	
	Now, note that $m(x)$ has a minimum where $m'(x)=0$, i.e., where $1-\frac{\rho p}{\tau}x^{p-1}=0$ giving $x=x_{b}=(\frac{\rho p}{\tau})^\frac{1}{1-p}$. Further, $m''(x)=\frac{(1-p)p\rho }{\tau}x^{p-2}>0$ for all $\tau^{1/p}\geqslant x>0$. So $m$ has a global minimum of value $h_{b}=x_{b}+\rho-\rho\frac{x_{b}^p}{\tau}$. Thus \eqref{mx} no solution unless $y\geqslant h_{b}$, see Fig. \ref{FUCT3}. Further, $m'(x)<0$ for $x<x_{b}$ ($m$ is strictly decreasing for $x<x_{b}$), while $m'(x)>0$ for $x>x_{b}$ ($m$ is strictly increasing for $x>x_{b}$). So \eqref{mx} has at least one solution for $y\geqslant h_{b}$. Once again, for the illustration see Fig. \ref{FUCT3}.
	\begin{figure}[!tbh]
		\centering
		\includegraphics[width=0.6\linewidth]{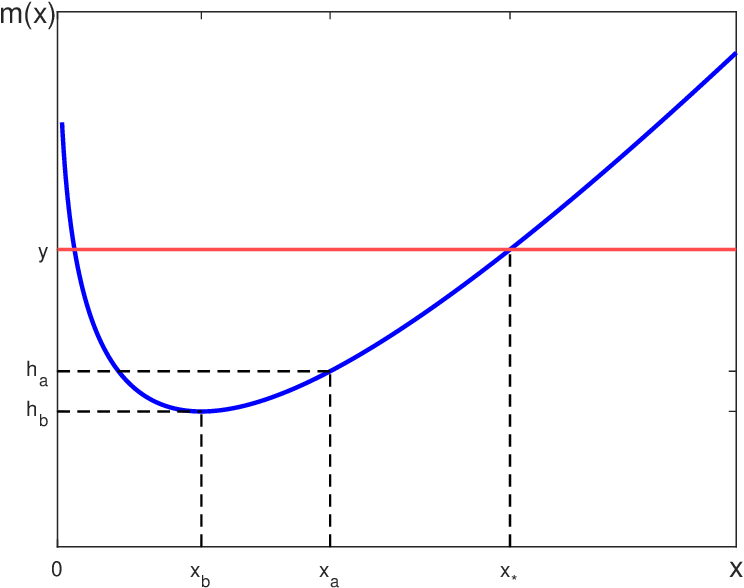}
		\caption {Plot for $m(x)$ for $x>0$. }
		\label{FUCT3}
	\end{figure}
	Further, suppose $x_{\ast}$ is a solution to \eqref{mx}. Then the corresponding value of $g$ is:
	\begin{align*}
		g(x_{\ast})&=\frac{1}{2}y^2-yx_{\ast}+\frac{1}{2}x_{\ast}^2+\rho( x_{\ast}-\frac{ x_{\ast}^{1+p}}{(1+p)\tau})
		\\&=g(0)+\frac{1}{2}x_{\ast}^{1+p}(x_{\ast}^{1-p}+2(\rho-y)x_{\ast}^{-p}-\frac{2\rho}{(1+p)\tau})
		\\&=g(0)+\frac{1}{2}x_{\ast}^{1+p}(\frac{2\rho p}{(1+p)\tau}-x_{\ast}^{1-p})
		\\&=g(0)+\frac{1}{2}x_{\ast}^{1+p}(x_{a}^{1-p}-x_{\ast}^{1-p}),
	\end{align*} 
	which implies that $g(x_{\ast})\leqslant g(0)$ if $x_{\ast}\geqslant x_{a}$. So for a non-zero solution we are only interested in $x_{\ast}>x_{a}$, and $x_{a}>x_{b}$, it can be seen in Fig. \ref{FUCT3}, this is the larger solution to \eqref{mx}. 
	
	Lastly, as $m(x)$ is strictly increasing for $x>x_{b}$ we have that $x_{\ast}\geqslant x_{a}$ is equivalent to $y\geqslant h_{a}$. This can also be seen in Fig. \ref{FUCT3}. We thus conclude that the global minimizer of $g(x)$ is as stated in the theorem.
	
	The iterations using \eqref{pro2} only need to be applied when $\tau^{1/p}>\rvert y\rvert> h_{a}$. For the case $\rvert y\rvert= h_{a}$ we already have that $x_{\ast}=x_{a}$. As a result, we take $\tau^{1/p}>\rvert y\rvert> h_{a}$ and firstly show that:
	\begin{equation}
		\alpha([x_{a},\rvert y\rvert])\subset[x_{a},\rvert y\rvert].\label{prox3}
	\end{equation}
	
	Let $x\in[x_{a},\rvert y\rvert]$. Then as $x\geqslant x_{a}>0$, by definition of $\alpha$ we have $\alpha(x)<\rvert y\rvert$, and by definition of $h_{a}$ we also have $h_a\geqslant x_a+\rho-\rho x^{p}/\tau$. This last inequality and $\rvert y\rvert>h_a$ imply $\rvert y\rvert> x_a+\rho-\rho x^{p}/\tau$, which by re-arrangement and definition of $\alpha$ gives $\alpha(x)>x_a$. Therefore, $x_a<\alpha(x)<\rvert y\rvert$, which proves \eqref{prox3}.
	
	Next, we show that $\alpha$ is a contraction mapping on $[x_{a},\rvert y\rvert]$. By differentiating we have $\alpha'(x)=p\rho x^{p-1}/\tau$, and obviously $\alpha'(x)>0$ as $x\geqslant x_a>0$ i.e., $\alpha'$ is positive on $[x_{a},\rvert y\rvert]$. Then by re-arranging $\alpha'$ it can easily be deduced that: $\alpha'(x)<1$ iff $x>x_b$. Since $x\geqslant x_a>x_b$ we have that $\alpha'(x)<1$. Therefore, $\alpha$ is a contraction on $[x_{a},\rvert y\rvert]$.
	
	Lastly, as $\mathbb{R}$ is complete and $[x_{a},\rvert y\rvert]\subset\mathbb{R}$ is closed implies $[x_{a},\rvert y\rvert]$ is also complete. Then by the standard Banach fixed point theorem: $\alpha$ admits one and only one fixed point $x_{\dagger}\in[x_{a},\rvert y\rvert]$, i.e., $x_{\dagger}=\rvert y\rvert-\rho+\rho x_{\dagger}^{p}/\tau$. Since $x_{\star}\in(x_{a},\rvert y\rvert)\subset[x_{a},\rvert y\rvert]$ satisfies the fixed point equation we must have $x_{\dagger}=x_{\star}$.
\end{proof}

For the tensor case, the tensor $p$-th order $\tau$ norm is proposed, and its properties are analyzed.

\begin{definition}[Tensor $p$-th order $\tau$ norm]
	The tensor $p$-th order $\tau$ norm of $\mathcal{B}\in\mathbb{R}^{\mathit{J}_{1}\times\mathit{J}_{2}\times\mathit{J}_{3}}$, denoted by $\|\mathcal{B}\|_{\tau p}$, is defined as follows:
	\begin{equation}
		\|\mathcal{B}\|_{\tau p}=
		\frac{1}{\mathit{J}_{3}}\sum_{j=1}^{\mathit{J}_{3}}\|\bar{\mathcal{B}}^{(j)}\|_{\tau p}=
		\frac{1}{\mathit{J}_{3}}\sum_{j=1}^{\mathit{J}_{3}}\sum_{i=1}^{R}\psi_{\tau,p}(\sigma_{i}(\bar{\mathcal{B}}^{(j)})),\label{tptn}
	\end{equation}
	where $\tau>1, 1> p>0$; $\bar{\mathcal{B}}^{(j)}$ is the $j$-th frontal slice of $\bar{\mathcal{B}}$, with $\bar{\mathcal{B}}=\text{\tt fft}(\mathcal{B},[~],3)$; $R=\min(\mathit{J}_{1},\mathit{J}_{2})$; $\sigma_{i}(\bar{\mathcal{B}}^{(j)})$ is the $i$-th singular value of $\bar{\mathcal{B}}$.
\end{definition}
Unlike the tensor nuclear norm, the tensor $p$-th order $\tau$ norm is not a true norm, as it does not satisfy homogeneity. Therefore, the term `` norm " is used merely as a naming convention. Below, some important properties satisfied by the tensor $p$-th order $\tau$ norm are presented.
\begin{theorem}\label{PROPER}
	The tensor $p$-th order $\tau$ norm defined in \eqref{tptn} satisfies the following properties:
	\\\textbf{1. Non-negativity}: The tensor $p$-th order $\tau$ norm is non-negative, i.e., $\|\mathcal{B}\|_{\tau p}$ $\geqslant0$. The equality holds if and only if $\mathcal{B}$ is the null tensor. 
	\\\textbf{2. Unitary invariance}: The tensor $p$-th order $\tau$ norm is unitary invariant, i.e., $\|\mathcal{U}\ast\mathcal{B}\ast\mathcal{V}\|_{\tau p}=\|\mathcal{B}\|_{\tau p}$, for unitary tensor $\mathcal{U}$ and $\mathcal{V}$.
	\\\textbf{3. Triangle inequality}: Suppose that $\mathcal{C}\in\mathbb{R}^{\mathit{J}_{1}\times\mathit{J}_{2}\times\mathit{J}_{3}}$ and $\mathcal{D}\in\mathbb{R}^{\mathit{J}_{1}\times\mathit{J}_{2}\times\mathit{J}_{3}}$ are two arbitrary tensors. Then, the following propertie hold:
	$\|\mathcal{C}-\mathcal{D}\|_{\tau p}\geqslant\|\mathcal{C}\|_{\tau p}-\|\mathcal{D}\|_{\tau p}$.
\end{theorem}
Before we proof Theorem \ref{PROPER}, we first present a lemma.

\begin{lemma}[\cite{41}, Theorem 1] \label{mtri}Suppose that $C$ and $D$ are two same-sized matrices and	that $\psi(x)$ satisfies Theorem \ref{propert} \textbf{(i)}. Then $\| C + D\|_{\tau p} \leqslant \| C\|_{\tau p} + \| D\|_{\tau p} $.
\end{lemma}
Then we proceed to prove Theorem \ref{PROPER}.
\begin{proof}
	\textbf{1.} Since $\|\mathcal{B}\|_{\tau p}$ is the sum of non-negative functions $\psi_{\tau,p}(\sigma_{i}(\bar{\mathcal{B}}^{(j)}))$, $\|\mathcal{B}\|_{\tau p}$ is non-negative, i.e., $\|\mathcal{B}\|_{\tau p}\geqslant0$.
	\\\textbf{2.} Considering the definition of the tensor $p$-th order $\tau$ norm, $\|\mathcal{B}\|_{\tau p}$ can be expressed as follows:
	\begin{equation*}
		\|\mathcal{B}\|_{\tau p}=\frac{1}{\mathit{J}_{3}}\sum_{j=1}^{\mathit{J}_{3}}\text{\tt Tr}(\psi_{\tau,p}(\sqrt{\bar{\mathcal{B}}^{(j)T}\bar{\mathcal{B}}^{(j)}})),
	\end{equation*}
	where ${\tt Tr}(\cdot)$ denotes the trace operator, $\psi_{\tau,p}(\cdot)$ means to take the MPCP function of each element. Next, consider $\|\mathcal{U}\ast \mathcal{B}\ast\mathcal{V}\|_{\tau p}$, where $\mathcal{U}$ and $\mathcal{V}$ are unitary tensors:
	\begin{equation*}
		\|\mathcal{U}\ast\mathcal{B}\ast\mathcal{V}\|_{\tau p}
		=\frac{1}{\mathit{J}_{3}}\sum_{j=1}^{\mathit{J}_{3}}\text{\tt Tr}(\psi_{\tau,p}(\sqrt{(\bar{\mathcal{U}}^{(j)}\bar{\mathcal{B}}^{(j)}\bar{\mathcal{V}}^{(j)})^{T}\bar{\mathcal{U}}^{(j)}\bar{\mathcal{B}}^{(j)}\bar{\mathcal{V}}^{(j)}})).
	\end{equation*}
	
	Properties of the tensor generated by performing DFT show that $\bar{\mathcal{U}}^{(j)}$ and $\bar{\mathcal{V}}^{(j)}$ are unitary matrix. Then, we get the following formula:
	\begin{equation*}
		\|\mathcal{U}\ast\mathcal{B}\ast\mathcal{V}\|_{\tau p}=\|\mathcal{B}\|_{\tau p}.
	\end{equation*}
	
	This establishes the invariance of the tensor $p$-th order $\tau$ norm to unitary transformations.
	\\\textbf{3.} Denote $\mathcal{A}=\mathcal{C}-\mathcal{D}$. By Lemma \ref{mtri}, we obtain that
	\begin{align*}{}
		\|\mathcal{C}\|_{\tau p}&=\|\mathcal{A}+\mathcal{D}\|_{\tau p}
		\\&=\frac{1}{\mathit{J}_{3}}\sum_{j=1}^{\mathit{J}_{3}}\|\bar{\mathcal{A}}^{(j)}+\bar{\mathcal{D}}^{(j)}\|_{\tau p}
		\\&\leqslant\frac{1}{\mathit{J}_{3}}\sum_{j=1}^{\mathit{J}_{3}}\|\bar{\mathcal{A}}^{(j)}\|_{\tau p}+\frac{1}{\mathit{J}_{3}}\sum_{j=1}^{\mathit{J}_{3}}\|\bar{\mathcal{D}}^{(j)}\|_{\tau p}
		\\&=\|\mathcal{A}\|_{\tau p}+\|\mathcal{D}\|_{\tau p},
	\end{align*}
	which implies that $\|\mathcal{C}-\mathcal{D}\|_{\tau p}\geqslant\|\mathcal{C}\|_{\tau p}-\|\mathcal{D}\|_{\tau p}$.
\end{proof}

Additionally, to facilitate the solution of subsequent algorithms, the proximal operator for the tensor $p$-th order $\tau$ norm is proposed.
\begin{theorem}[Proximal operator for the tensor $p$-th order $\tau$ norm]\label{thTPTN}
	Consider the tensor $p$-th order $\tau$ norm given in \eqref{tptn}. Its proximal operator is denoted by $\mathit{S}_{\rho}:\mathbb{R}^{\mathit{J}_{1}\times\mathit{J}_{2}\times\mathit{J}_{3}}\to\mathbb{R}^{\mathit{J}_{1}\times\mathit{J}_{2}\times\mathit{J}_{3}}$, $0<p<1, \tau>1$, $R=\min\{\mathit{J}_{1},\mathit{J}_{2}\}$ and defined as follows:
	\begin{equation}
		\mathit{S}_{\rho}(\mathcal{Y})=\arg\min_{\mathcal{B}} \frac{1}{2}\|\mathcal{B}-\mathcal{Y}\|_{F}^{2}+\rho\|\mathcal{B}\|_{\tau p}  \label{prox}
	\end{equation}
	is given by
	\begin{equation}
		\mathit{S}_{\rho}(\mathcal{Y})=\mathcal{M}\ast\mathcal{S}_{1}\ast\mathcal{Q}^{H},
		\label{opewtgn}
	\end{equation}where $\mathcal{M}$ and $\mathcal{Q}$ are derived from the t-SVD of $\mathcal{Y}=\mathcal{M}\ast\mathcal{S}_{2}\ast\mathcal{Q}^{H}$. Additionally, the $(k,k)$-th elements of the $j$-th frontal slice of $\bar{\mathcal{S}}_{1}$ and $\bar{\mathcal{S}}_{2}$ are $\sigma_{k}(\bar{\mathcal{B}}^{(j)})$ and $\sigma_{k}(\bar{\mathcal{Y}}^{(j)})$, i.e., $\bar{\mathcal{S}}^{(j)}_{1}{(k,k)}=\sigma_{k}(\bar{\mathcal{B}}^{(j)})$ and $\bar{\mathcal{S}}^{(j)}_{2}{(k,k)}=\sigma_{k}(\bar{\mathcal{Y}}^{(j)})$, respectively. More importantly, $\sigma_{k}(\bar{\mathcal{B}}^{(j)})$ and $\sigma_{k}(\bar{\mathcal{Y}}^{(j)})$ have the following relationship $\sigma_{k}(\bar{\mathcal{B}}^{(j)})=\mathit{P}_{\rho}(\sigma_{k}(\bar{\mathcal{Y}}^{(j)}))$. 
\end{theorem}
\begin{proof}
	Let $\mathcal{Y}=\mathcal{M}\ast\mathcal{S}_{2}\ast\mathcal{Q}^{H}$ and $\mathcal{B}=\mathcal{U}\ast\mathcal{S}_{1}\ast\mathcal{V}^{H}$ be the t-SVD of $\mathcal{Y}$ and $\mathcal{B}$, respectively. Consider
	\begin{equation}\label{proflc}
		\begin{aligned}
			\mathit{S}_{\rho}(\mathcal{Y})&=\arg\min_{\mathcal{B}}\frac{1}{2}\|\mathcal{B}-\mathcal{Y}\|_{F}^{2}+\rho\|\mathcal{B}\|_{\tau p}
			\\&=\arg\min_{\mathcal{B}}\frac{1}{2}\|\mathcal{U}\ast\mathcal{S}_{1}\ast\mathcal{V}^{H}-\mathcal{M}\ast\mathcal{S}_{2}\ast\mathcal{Q}^{H}\|_{F}^{2}+\rho\|\mathcal{B}\|_{\tau p}
			\\&=\arg\min_{\bar{\mathcal{B}}}\frac{1}{\mathit{J}_{3}}\sum_{j=1}^{\mathit{J}_{3}}(\frac{1}{2}\|\bar{\mathcal{U}}^{(j)}\ast\bar{\mathcal{S}_{1}}^{(j)}\ast\bar{\mathcal{V}}^{(j)H}-\bar{\mathcal{M}}^{(j)}\ast\bar{\mathcal{S}_{2}}^{(j)}\ast\bar{\mathcal{Q}}^{(j)H}\|_{F}^{2}+\rho\|\bar{\mathcal{B}}^{(j)}\|_{\tau p}).
		\end{aligned}	
	\end{equation}
	It can be found that \eqref{proflc} is separable and can be divided into $\mathit{J}_{3}$ sub-problems. For the $j$-th sub-problem:
	\begin{align*}
		&\arg\min_{\bar{\mathcal{B}}^{(j)}}\frac{1}{2}\|\bar{\mathcal{U}}^{(j)}\ast\bar{\mathcal{S}_{1}}^{(j)}\ast\bar{\mathcal{V}}^{(j)H}-\bar{\mathcal{M}}^{(j)}\ast\bar{\mathcal{S}_{2}}^{(j)}\ast\bar{\mathcal{Q}}^{(j)H}\|_{F}^{2}+\rho\|\bar{\mathcal{B}}^{(j)}\|_{\tau p}\nonumber
		\\&=\arg\min_{\bar{\mathcal{B}}^{(j)}}\frac{1}{2}\text{\tt Tr}(\bar{\mathcal{S}_{1}}^{(j)}\bar{\mathcal{S}_{1}}^{(j)H})+\frac{1}{2}\text{\tt Tr}(\bar{\mathcal{S}_{2}}^{(j)}\bar{\mathcal{S}_{2}}^{(j)H})-\text{\tt Tr}(\bar{\mathcal{B}}^{(j)H}\bar{\mathcal{Y}}^{(j)})+\rho\|\bar{\mathcal{B}}^{(j)}\|_{\tau p}.
	\end{align*}
	Invoking von Neumann's trace inequality \cite{mirsky1975trace}, we can write
	\begin{equation}\label{proftp}
		\begin{aligned}
			&\arg\min_{\bar{\mathcal{B}}^{(j)}}\frac{1}{2}\|\bar{\mathcal{U}}^{(j)}\ast\bar{\mathcal{S}_{1}}^{(j)}\ast\bar{\mathcal{V}}^{(j)H}-\bar{\mathcal{M}}^{(j)}\ast\bar{\mathcal{S}_{2}}^{(j)}\ast\bar{\mathcal{Q}}^{(j)H}\|_{F}^{2}+\rho\|\bar{\mathcal{B}}^{(j)}\|_{\tau p}
			\\&\geqslant\arg\min_{\bar{\mathcal{B}}^{(j)}}\frac{1}{2}{\tt Tr}(\bar{\mathcal{S}_{1}}^{(j)}\bar{\mathcal{S}_{1}}^{(j)H})+\frac{1}{2}{\tt Tr}(\bar{\mathcal{S}_{2}}^{(j)}\bar{\mathcal{S}_{2}}^{(j)H})-{\tt Tr}(\bar{\mathcal{S}_{2}}^{(j)}\bar{\mathcal{S}_{1}}^{(j)H})+\rho\|\bar{\mathcal{B}}^{(j)}\|_{\tau p}
			\\&=\sum_{k=1}^{R}\arg\min_{\sigma_{k}(\bar{\mathcal{B}}^{(j)})}\frac{1}{2}(\sigma_{k}(\bar{\mathcal{B}}^{(j)})-\sigma_{k}(\bar{\mathcal{Y}}^{(j)}))^{2}+\rho\psi_{\tau,p}(\sigma_{k}(\bar{\mathcal{B}}^{(j)})).
		\end{aligned}	
	\end{equation}
	The equality holds when $\bar{\mathcal{U}}^{(j)}=\bar{\mathcal{M}}^{(j)}$ and $\bar{\mathcal{V}}^{(j)}=\bar{\mathcal{Q}}^{(j)}$. By Theorem \ref{tMPCP}, the optimal solution of \eqref{proftp} can be obtained through $\sigma_{k}(\bar{\mathcal{B}}^{(j)})=\mathit{P}_{\rho}(\sigma_{k}(\bar{\mathcal{Y}}^{(j)}))$. 
\end{proof}
\begin{remark}
	The proximal operator for the tensor $p$-th order $\tau$ norm can also degenerate into a matrix form. When $\mathit{J}_3 = 1$, the proximal operator degenerates to the matrix case, where it is only necessary to compute the singular values of the matrix. This indicates that Theorem \ref{thTPTN} applies not only to the third-order tensor case but also to the second-order matrix case.
\end{remark}

\section{MPCP-based models and solving algorithm}
The LRTC problem aims at estimating the missing elements from an incomplete observed tensor. Considering an $N$th-order observed tensor $\mathcal{O}\in\mathbb{R}^{\mathit{J}_{1}\times\mathit{J}_{2}\times\cdots\times\mathit{J}_{N}}$, the LRTC problem is
\begin{equation*}
	\min_{\mathcal{B}}\,{\tt rank}(\mathcal{B})~~s.t.~~\mathcal{P}_{\Omega}(\mathcal{B})=\mathcal{P}_{\Omega}(\mathcal{O}),
\end{equation*}
where $\mathcal{B}$ is the underlying tensor; the operator $\mathcal{P}_{\Omega}(\mathcal{B})$ function as a projection operator, retaining only the entries of $\mathcal{B}$ that belong to $\Omega$ while setting all others to zero. 

First, we use the $N$-tubal rank to represent the tensor rank function, then the above equation is transformed into the following form:
\begin{equation*}
	\min_{\mathcal{B}}\sum_{q=1}^{N(N-1)/2}\beta_{q}  {\tt rank}_{t}(\mathcal{B}_{<q>})~~s.t.~~ \mathcal{P}_{\Omega}(\mathcal{B})=\mathcal{P}_{\Omega}(\mathcal{O}).
\end{equation*}

Next, for each rank in $N$-tubal rank, we apply the MPCP function as the non-convex relaxation, then the proposed MPCP-based LRTC model is formulated as follow
\begin{equation}
	\min_{\mathcal{B}}\sum_{q=1}^{N(N-1)/2}\beta_{q}  \|\mathcal{B}_{<q>}\|_{\tau p}~~s.t.~~ \mathcal{P}_{\Omega}(\mathcal{B})=\mathcal{P}_{\Omega}(\mathcal{O}).\label{LRTCMPCP}
\end{equation}

Problem \eqref{LRTCMPCP} is difficult to optimize due to the application of multiple tensor $p$-th order $\tau$ norms to a tensor $\mathcal{B}$. Therefore, we introduce $N(N-1)/2$ auxiliary tensors $\{\mathcal{M}_{q}\}_{q=1}^{N(N-1)/2}$ and the indicator
function $\Phi(\mathcal{B})$ to simplify this problem:
\begin{equation}
	\min_{\mathcal{B},\{\mathcal{M}_{q}\}_{q=1}^{N(N-1)/2}}\sum_{q=1}^{N(N-1)/2}\beta_{q} \|\mathcal{M}_{q<q>}\|_{\tau p}+\Phi(\mathcal{B})~~s.t.~~ \{\mathcal{M}_{q}=\mathcal{B}\}_{q=1}^{N(N-1)/2},\label{LRTCMPCP1}
\end{equation}
where 
\begin{equation*}
	\Phi(\mathcal{B}):=\begin{cases}
		0,~~&\mathcal{P}_{\Omega}(\mathcal{B})=\mathcal{P}_{\Omega}(\mathcal{O}),
		\\+\infty,&\text{otherwise}.
	\end{cases}
\end{equation*}
Next, we use the ADMM \cite{683201156} algorithm to solve problem \eqref{LRTCMPCP1}, and its augmented Lagrangian function is as follows:
\begin{equation}
	L(\mathcal{B},\mathcal{M}_{q},\mathcal{T})
	=\sum_{q=1}^{N(N-1)/2}\beta_{q} \|\mathcal{M}_{q<q>}\|_{\tau p}+<\mathcal{B}-\mathcal{M}_{q},\mathcal{T}_{q}>+\frac{\rho_{q}}{2}\|\mathcal{B}-\mathcal{M}_{q}\|_{F}^{2}+\Phi(\mathcal{B}), \label{MPCP2}
\end{equation}
where $\{\mathcal{T}_{q}\}_{q=1}^{N(N-1)/2}$ is the set of Lagrange multiplier tensors; $\{\rho_{q}\}_{q=1}^{N(N-1)/2}>0$ are the augmented Lagrangian parameters; and $\beta_{q}\geqslant0$ are weights, with $\sum_{q=1}^{N(N-1)/2}\beta_{q}=1$. Accordingly, we iteratively solved the following set of subproblems within the ADMM framework:
\begin{equation}\label{UPT}
	\begin{cases}
		\mathcal{B}^{k+1}=\arg\min_{\mathcal{B}}L(\mathcal{B},\mathcal{M}_{q}^{k},\mathcal{T}^{k}),
		\\\mathcal{M}_{q}^{k+1}=\arg\min_{\mathcal{M}_{q}}L(\mathcal{B}^{k+1},\mathcal{M}_{q},\mathcal{T}^{k}),
		\\\mathcal{T}_{q}^{k+1}=\mathcal{T}_{q}^{k}+\rho_{q}^{k}(\mathcal{B}^{k+1}-\mathcal{M}_{q}^{k+1}),
	\end{cases} 
\end{equation}
where $k$ represents the iteration number. Below, we provide the detailed solutions for the variables $\mathcal{B}$ and $\mathcal{M}_{q}$.

\textbf{Update $\mathcal{B}$}: The $\mathcal{B}$-subproblem is a quadratic optimization problem as follows:
\begin{align}
	\mathcal{B}^{k+1}&=\arg\min_{\mathcal{B}}\Phi(\mathcal{B})+\sum_{q=1}^{N(N-1)/2}\frac{\rho_{q}^{k}}{2}\|\mathcal{B}-\mathcal{M}_{q}^{k}+\frac{\mathcal{T}_{q}^{k}}{\rho_{q}^{k}}\|_{F}^{2}\label{UPBB}
	\\&=\mathcal{P}_{\Omega^{c}}((\sum_{q}\rho_{q}^{k}\mathcal{M}_{q}^{k}-\mathcal{T}_{q}^{k})/\sum_{q}\rho_{q}^{k})+\mathcal{P}_{\Omega}(\mathcal{O}),\label{UPB}
\end{align}
where $\Omega^{c}$ is the complement of $\Omega$.

\textbf{Update $\mathcal{M}_{q}$}: Fix other variables, and the corresponding optimization is as follows: 
\begin{equation}
	\mathcal{M}_{q}^{k+1}=\arg\min_{\mathcal{M}_{q}}\beta_{q} \|\mathcal{M}_{q<q>}\|_{\tau p}+\frac{\rho_{q}^{k}}{2}\|\mathcal{B}^{k+1}-\mathcal{M}_{q}+\frac{\mathcal{T}_{q}^{k}}{\rho_{q}^{k}}\|_{F}^{2}.\label{UPMM}
\end{equation}

Recalling Theorem \ref{thTPTN}, the solution to the above optimization is given by:
\begin{equation}
	\mathcal{M}_{q}^{k+1}=\mathit{S}_{\rho_{q}^{k}/\beta_{q}}(\mathcal{B}^{k+1}+\frac{\mathcal{T}_{q}^{k}}{\rho_{q}^{k}}).\label{UPM}
\end{equation}

The optimization steps of the MPCP formulation are outlined in Algorithm \ref{ATC}. The primary computational cost per-iteration arises from the update of $\mathcal{M}_{q}$, which requires the computation of t-SVD. The per-iteration complexity is $O(Y(\sum_{q}[\log(y_{q})+\min(\mathit{J}_{q_{1}},\mathit{J}_{q_{2}})]))$, where $Y=\prod_{j=1}^{N}\mathit{J}_{j}$ and $y_{q}=Y/(\mathit{J}_{q_{1}}\mathit{J}_{q_{2}})$. This matches the computational complexity presented in \cite{2020170}.
\begin{algorithm}[!h]
	\caption{MPCP-LRTC}
	\hspace*{0.02in} {\bf Input:}
	An incomplete tensor $\mathcal{O}$, the index set of the known elements $\Omega$, convergence criteria $\epsilon=10^{-4}$, maximum iteration number $K$. \\
	\hspace*{0.02in} {\bf Initialization:}
	$\mathcal{B}^{0}=\mathcal{P}_{\Omega}(\mathcal{O})$, $\mathcal{Y}_{q}^{0}=\mathcal{B}^{0}$, $\rho_{q}^{0}>0$, $\mu=1.05$.
	\begin{algorithmic}
		\While{not converged and $k<K$}
		\State Updating $\mathcal{B}^{k}$ via \eqref{UPB};
		\State Updating $\mathcal{M}_{q}^{k}$ via \eqref{UPM};
		\State Updating the multipliers $\mathcal{T}_{q}^{k}$ via \eqref{UPT};
		\State $\rho_{q}^{k}=\mu\rho_{q}^{k-1}$, $k=k+1$;
		\State Check the convergence conditions $\|\mathcal{B}^{k+1}-\mathcal{B}^{k}\|^{2}_{F}/\|\mathcal{B}^{k}\|^{2}_{F}\leqslant\epsilon$.
		\EndWhile
		\State \Return $\mathcal{B}^{k+1}$.
	\end{algorithmic}
	\hspace*{0.02in} {\bf Output:} 
	Completed tensor $\mathcal{B}$.
	\label{ATC}\end{algorithm}

\section{Convergence analysis}
In this section, to rigorously establish the theoretical convergence analysis of the proposed Algorithm \ref{ATC}, the following lemma is first introduced, which serves as a foundational result for the subsequent proofs and discussions.
\begin{lemma}[Proposition 6.2 of \cite{lewis2005nonsmooth}]\label{le5.1}
	Suppose $F:\mathbb{R}^{\mathit{I}_{1}\times \mathit{I}_{2}}\to\mathbb{R}$ is represented as $F(X)=f\circ\sigma(X)$, where $X\in\mathbb{R}^{\mathit{I}_{1}\times \mathit{I}_{2}}$ with SVD : $X=U\text{diag}(\sigma_{1},\dots,\sigma_{R})V^{T}$, $R=\min\{\mathit{I}_{1},\mathit{I}_{2}\}$, and $f$ is differentiable. The gradient of $F(X)$ at $X$ is $	\frac{\partial F(X)}{\partial X}=U\text{diag}(\theta)V^{T}$,
	$\theta=\frac{\partial f(y)}{\partial y}\lvert_{y=\sigma(X)}$.
\end{lemma}
\begin{lemma}\label{le5.2}
	The sequence $\{\mathcal{T}_{q}^{k}\}_{q=1}^{N(N-1)/2}$ generated by Algorithm \ref{ATC} are bounded.
\end{lemma}
\begin{proof}
	From the Lemma \ref{le5.1}, for a tensor $\mathcal{Q}\in\mathbb{R}^{\mathit{J}_{1}\times\mathit{J}_{2}\times\mathit{J}_{3}}$, we have
	\begin{equation*}
		\frac{\partial \|\bar{\mathcal{Q}}^{(j_{3})}\|_{\tau p}}{\partial \bar{\mathcal{Q}}^{(j_{3})}}=\bar{\mathcal{U}}^{(j_{3})}\text{diag}\left( \frac{\partial \psi_{\tau,p}(\sigma_{1}(\bar{\mathcal{Q}}^{(j_{3})}))}{\partial \sigma_{1}(\bar{\mathcal{Q}}^{(j_{3})})},\cdots,\frac{\partial \psi_{\tau,p}(\sigma_{R}(\bar{\mathcal{Q}}^{(j_{3})}))}{\partial \sigma_{R}(\bar{\mathcal{Q}}^{(j_{3})})}\right) \bar{\mathcal{V}}^{(j_{3})T},
	\end{equation*}
	where $R=\min\{\mathit{J}_{1},\mathit{J}_{2}\}$. Based on the proof of Theorem \ref{propert}, it is not difficult to obtain the following expression:
	\begin{equation*}
		\frac{\partial \psi_{\tau,p}(\sigma_{i}(\bar{\mathcal{Q}}^{(j_{3})}))}{\partial \sigma_{i}(\bar{\mathcal{Q}}^{(j_{3})})}\leqslant C,~\forall~i=1,\dots,R,~\Rightarrow~\|\frac{\partial \|\bar{\mathcal{Q}}^{(j_{3})}\|_{\tau p}}{\partial \bar{\mathcal{Q}}^{(j_{3})}}\|_{F}^2\leqslant~RC.
	\end{equation*}
	Therefore,
	\begin{equation*}
		\frac{\partial \|\bar{\mathcal{Q}}\|_{\tau p}}{\partial \bar{\mathcal{Q}}}=\left[ \frac{\partial \|\bar{\mathcal{Q}}^{(1)}\|_{\tau p}}{\partial \bar{\mathcal{Q}}^{(1)}}\lvert\cdots\lvert\frac{\partial \|\bar{\mathcal{Q}}^{(\mathit{J}_{3})}\|_{\tau p}}{\partial \bar{\mathcal{Q}}^{(\mathit{J}_{3})}} \right] 
	\end{equation*}
	is also bounded. From $\bar{\mathcal{Q}}=\text{\tt fft}(\mathcal{Q},[~],3)=\mathcal{Q}\times_{3}F$ and the chain rule, we conclude that 
	\begin{equation*}
		\|\partial_{\mathcal{Q}}\|\mathcal{Q}\|_{\tau p}\|_F^2=\|\frac{\partial \|\mathcal{Q}\|_{\tau p}}{\partial \mathcal{Q}}\|_{F}^2=\|\frac{\partial \|\mathcal{Q}\|_{\tau p}}{\partial \bar{\mathcal{Q}}}\times_{3}F^{H}\|_{F}^2\leqslant\mathit{J}_3 RC
	\end{equation*}
	is bounded. From the first-order optimality of \eqref{UPMM}, in $\mathcal{M}_{q<q>}$, we can deduce
	\begin{align*}
		\mathbf{0}&\in\beta_{q} \partial_{\mathcal{M}^{k+1}_{q<q>}}\|\mathcal{M}^{k+1}_{q<q>}\|_{\tau p}-\rho_{q}^{k}(\mathcal{B}^{k+1}-\mathcal{M}^{k+1}_{q}+\frac{\mathcal{T}_{q}^{k}}{\rho_{q}^{k}})
		\\&=\beta_{q} \partial_{\mathcal{M}^{k+1}_{q<q>}}\|\mathcal{M}^{k+1}_{q<q>}\|_{\tau p}-\mathcal{T}_{q}^{k+1}.
	\end{align*}
	Therefore, the sequence $\{\mathcal{T}_{q}^{k}\}_{q=1}^{N(N-1)/2}$ are bounded.
\end{proof}
\begin{theorem}[convergence analysis]\label{th5.1}
	The sequences $\{\mathcal{B}^{k}\}$ and $\{\mathcal{M}_{q}^{k}\}_{q=1}^{N(N-1)/2}$ generated by Algorithm \ref{ATC} are Cauchy sequences and convergent.
\end{theorem}
\begin{proof}
	First, regarding to $\mathcal{B}^{k+1}$, from \eqref{UPT}, we have 
	\begin{equation*}
		\mathcal{T}_{q}^{k+1}=\mathcal{T}_{q}^{k}+\rho_{q}^{k}(\mathcal{B}^{k+1}-\mathcal{M}_{q}^{k+1})\Rightarrow \mathcal{B}^{k+1}=\mathcal{M}_{q}^{k+1}+(\mathcal{T}_{q}^{k+1}-\mathcal{T}_{q}^{k})/\rho_{q}^{k}.
	\end{equation*}
	
	To begin, we observe that  $\mathcal{P}_{\Omega}(\mathcal{B}^{k+1}) = \mathcal{P}_{\Omega}(\mathcal{B}^{k})$, which leads to the following deduction:
	\begin{align*}
		\|\mathcal{B}^{k+1}-\mathcal{B}^{k}\|_{F}
		&=\|\mathcal{P}_{\Omega^{c}}(\mathcal{B}^{k+1}-\mathcal{B}^{k})\|_{F}
		\\&=\|\mathcal{P}_{\Omega^{c}}(\mathcal{B}^{k+1}-\frac{\sum_{q}\rho_{q}^{k}}{\sum_{q}\rho_{q}^{k}}\mathcal{B}^{k})\|_{F}
		\\&=\|\mathcal{P}_{\Omega^{c}}(\mathcal{B}^{k+1}-\frac{\sum_{q}\rho_{q}^{k}(\mathcal{M}_{q}^{k}+(\mathcal{T}_{q}^{k}-\mathcal{T}_{q}^{k-1})/\rho_{q}^{k-1})}{\sum_{q}\rho_{q}^{k}})\|_{F}
		\\&=\|\mathcal{P}_{\Omega^{c}}(\frac{\sum_{q}\rho_{q}^{k}((\mathcal{M}_{q}^{k}-\mathcal{T}_{q}^{k}/\rho_{q}^{k})-(\mathcal{M}_{q}^{k}+(\mathcal{T}_{q}^{k}-\mathcal{T}_{q}^{k-1})/\rho_{q}^{k-1}))}{\sum_{q}\rho_{q}^{k}})\|_{F}
		\\&=\|\mathcal{P}_{\Omega^{c}}(\frac{\sum_{q}\rho_{q}^{k}(\mathcal{T}_{q}^{k}/\rho_{q}^{k}+(\mathcal{T}_{q}^{k}-\mathcal{T}_{q}^{k-1})/\rho_{q}^{k-1})}{\sum_{q}\rho_{q}^{k}})\|_{F}
		\\&\leqslant\|\frac{\sum_{q}\rho_{q}^{k}(\mathcal{T}_{q}^{k}/\rho_{q}^{k}+(\mathcal{T}_{q}^{k}-\mathcal{T}_{q}^{k-1})/\rho_{q}^{k-1})}{\sum_{q}\rho_{q}^{k}}\|_{F}
		\\&\leqslant\sum_{q}\|(1+\mu)\mathcal{T}_{q}^{k}/\rho_{q}^{k-1}-\mathcal{T}_{q}^{k-1}/\rho_{q}^{k-1}\|_{F}
		\\&\leqslant\sum_{q}\frac{1}{\rho_{q}^{0}\mu^{k-1}}(\|(1+\mu)\mathcal{T}_{q}^{k}\|_{F}+\|\mathcal{T}_{q}^{k-1}\|_{F}).
	\end{align*}
	
	It is noted that $\{\mathcal{T}_{q}^{k}\}_{q=1}^{N(N-1)/2}$ are bounded,  and since $\lim\limits_{k\to\infty}\frac{1}{\mu^{k}}=0$, we conclude that  $\lim\limits_{k\to\infty}\|\mathcal{B}^{k+1}-\mathcal{B}^{k}\|_{F}=0$. For any $n<m$, applying the triangle inequality again yields:
	\begin{align*}
		\|\mathcal{B}^{m}-\mathcal{B}^{n}\|_{F}
		&\leqslant\sum_{k=n}^{m}\|\mathcal{B}^{k+1}-\mathcal{B}^{k}\|_{F}
		\\&\leqslant\sum_{k=n}^{m}\sum_{q}\frac{1}{\rho_{q}^{0}\mu^{k-1}}(\|(1+\mu)\mathcal{T}_{q}^{k}\|_{F}+\|\mathcal{T}_{q}^{k-1}\|_{F})
		\\&\leqslant\sum_{k=n}^{m}\frac{C_1}{\mu^{k-1}},
	\end{align*}
	where $C_{1}$ is an upper bound for  $\sum_{q}1/\rho_{q}^{0}(\|(1+\mu)\mathcal{T}_{q}^{k}\|_{F}+\|\mathcal{T}_{q}^{k-1}\|_{F})$. 
	Thus, we conclude that $\lim\limits_{m,n\to\infty} \|\mathcal{B}^{m}-\mathcal{B}^{n}\|_{F}=0$. Therefore, $\{\mathcal{B}^{k}\}$ is a Cauchy sequence, which is convergent.
	
	Finally, for the sequence  $\{\mathcal{M}_{q}^{k}\}_{q=1}^{N(N-1)/2}$, the following estimates can be derived:
	\begin{align*}
		\|\mathcal{M}_{q}^{m}-\mathcal{M}_{q}^{n}\|_{F}
		&\leqslant\sum_{k=n}^{m}\|\mathcal{M}_{q}^{k+1}-\mathcal{M}_{q}^{k}\|_{F}
		\\&\leqslant\sum_{k=n}^{m}\|\mathcal{B}^{k+1}-(\mathcal{T}_{q}^{k+1}-\mathcal{T}_{q}^{k})/\rho_{q}^{k}-\mathcal{B}^{k}+(\mathcal{T}_{q}^{k}-\mathcal{T}_{q}^{k-1})/\rho_{l_1l_2}^{k-1}\|_{F}
		\\&\leqslant\sum_{k=n}^{m}\|\mathcal{B}^{k+1}-\mathcal{B}^{k}\|_{F}+\frac{1}{\rho_{q}^{k}}\|\mathcal{T}_{q}^{k+1}\|_{F}+\frac{1+\mu}{\rho_{q}^{k}}\|\mathcal{T}_{q}^{k}\|_{F}+\frac{1}{\rho_{q}^{k-1}}\|\mathcal{T}_{q}^{k-1}\|_{F}
		\\&\leqslant\sum_{k=n}^{m}\frac{C_1+C_2}{\mu^{k-1}},
	\end{align*}
	where $C_{2}$ is an upper bound for the following expression: $\frac{1}{\rho_{q}^{0}\mu}\|\mathcal{T}_{q}^{k+1}\|_{F}+\frac{1+\mu}{\rho_{q}^{0}\mu}\|\mathcal{T}_{q}^{k}\|_{F}+\frac{1}{\rho_{q}^{0}}\|\mathcal{T}_{q}^{k-1}\|_{F}$. Thus, we conclude that $\lim\limits_{m,n\to\infty} \|\mathcal{M}_{q}^{m}-\mathcal{M}_{q}^{n}\|_{F}=0$. Therefore, the sequence $\{\mathcal{M}_{q}^{k}\}_{q=1}^{N(N-1)/2}$ forms a Cauchy sequence, which implies its convergence.
\end{proof}
\begin{theorem}[Karush-Kuhn-Tuker conditions]
	Suppose that $\mathcal{N}^{k}=\{\mathcal{B}^{k},\mathcal{M}_{q}^{k},\mathcal{T}_{q}^{k}\}$ is generated from the proposed Algorithm \ref{ATC}, then the limit point $\mathcal{N}^{\star}=\{\mathcal{B}^{\star},\mathcal{M}_{q}^{\star},\mathcal{T}_{q}^{\star}\}$ is a stationary point, which is $\mathbf{0}\in\partial L(\mathcal{B}^{\star},\mathcal{M}_{q}^{\star},\mathcal{T}_{q}^{\star})$, or equivalently,
	\begin{equation*}
		\mathcal{B}^{\star}-\mathcal{M}_{q}^{\star}=\mathbf{0},~~
		\mathcal{T}^{\star}_{q<q>}\in\partial_{\mathcal{M}_{q<q>}}\|\mathcal{M}_{q<q>}^{\star}\|_{\tau p},~~	
		\sum_{q}\mathcal{T}_{q}^{\star}=\partial_{\mathcal{B}}\Phi(\mathcal{B}^{\star}).
	\end{equation*}
\end{theorem}
\begin{proof}
	From Lemma \ref{le5.2}, Theorem \ref{th5.1} and the Bolzano-Weierstrass theorem, there must be at least one accumulation point of the sequence $\{\mathcal{N}^{k}\}_{k=1}^{+\infty}$. We denote one of the points $\mathcal{N}^{\star}=\{\mathcal{B}^{\star},\mathcal{M}_{q}^{\star},\mathcal{T}_{q}^{\star}\}$. Without loss of generality, we assume $\{\mathcal{N}^{k}\}_{k=1}^{+\infty}$ converge to $\mathcal{N}^{\star}$. Since $(\mathcal{T}_{q}^{k+1}-\mathcal{T}_{q}^{k})/\rho_{q}^{k}=\mathcal{B}^{k+1}-\mathcal{M}_{q}^{k+1}$, we have
	\begin{equation*}
		\lim\limits_{k\to\infty}(\mathcal{B}^{k+1}-\mathcal{M}_{q}^{k+1})=\lim\limits_{k\to\infty}(\mathcal{T}_{q}^{k+1}-\mathcal{T}_{q}^{k})/\rho_{q}^{k}=\textbf{0}.
	\end{equation*}
	Then $\mathcal{B}^{\star}-\mathcal{M}_{q}^{\star}=\textbf{0}$ is obtained. From \eqref{UPMM}, we can deduce 
	\begin{equation*}
		\textbf{0}\in\beta_{q}\partial_{\mathcal{M}_{q<q>}}\|\mathcal{M}_{q<q>}^{k+1}\|_{\tau p}+\rho_{q}^{k}(\mathcal{M}^{k+1}_{q<q>}-\mathcal{B}^{k+1}_{q<q>})+\mathcal{T}^{k}_{q<q>}.
	\end{equation*}
	Then it follows that $\mathcal{T}^{\star}_{q<q>}\in\partial_{\mathcal{M}_{q<q>}}\|\mathcal{M}_{q<q>}^{\star}\|_{\tau p}$.
	
	Similarly, from \eqref{UPBB}, we have
	\begin{align*}
		\textbf{0}=\sum_{q}\mathcal{T}_{q}^{k}+\rho_{q}^{k}(\mathcal{B}^{k+1}-\mathcal{M}_{q}^{k})+\partial_{\mathcal{B}}\Phi(\mathcal{B}^{k+1}),
		\\\lim\limits_{k\to\infty}(\mathcal{T}_{q}^{k}+\rho_{q}^{k}(\mathcal{B}^{k+1}-\mathcal{M}_{q}^{k}))=(\mathcal{T}_{q}^{\star}+\rho_{q}^{\star}(\mathcal{B}^{\star}-\mathcal{M}_{q}^{\star}))=\mathcal{T}_{q}^{\star}.
	\end{align*}
	Then $\sum_{q}\mathcal{T}_{q}^{\star}=\partial_{\mathcal{B}}\Phi(\mathcal{B}^{\star})$, where
	\begin{equation*}
		\begin{cases}
			\partial_{\mathcal{B}}\Phi(\mathcal{B}^{\star})_{j_{1},j_{2},\cdots,j_{N}}=0,~&\mathcal{P}_{\Omega}(\mathcal{B}_{j_{1},j_{2},\cdots,j_{N}})\neq\mathcal{P}_{\Omega}(\mathcal{O}_{j_{1},j_{2},\cdots,j_{N}}),
			\\\partial_{\mathcal{B}}\Phi(\mathcal{B}^{\star})_{j_{1},j_{2},\cdots,j_{N}}=1,&\mathcal{P}_{\Omega}(\mathcal{B}_{j_{1},j_{2},\cdots,j_{N}})=\mathcal{P}_{\Omega}(\mathcal{O}_{j_{1},j_{2},\cdots,j_{N}}).
		\end{cases}
	\end{equation*} 
	Therefore $\{\mathcal{B}^{\star},\mathcal{M}_{q}^{\star},\mathcal{T}_{q}^{\star}\}$ satisfies the the Karush-Kuhn-Tuker conditions of the Lagrange function $ L(\mathcal{B},\mathcal{M}_{q},\mathcal{T}_{q})$.
\end{proof}

\section{Experiments}
The performance of the proposed MPCP method is evaluated using several quality metrics, including peak signal-to-noise rate (PSNR) value, structural similarity (SSIM) value \cite{1284395}, feature similarity (FSIM) value \cite{5705575}, and the erreur relative globale adimensionnelle de synth$\grave{e}$se (ERGAS) value \cite{2432002352}. Among these metrics, higher PSNR, SSIM, and FSIM values indicate better performance, while a lower ERGAS value is preferred. All experiments were conducted on a Windows 11 platform, using MATLAB (R2023b), with an 13th Gen Intel Core i5-13600K 3.50 GHz processor and 32 GB of RAM. Let $J^c$ and $J^r$ represent the complete and reference images, respectively, while $J_1$ and $J_2$ denote the spatial dimensions of the image.

\textbf{1.} PSNR is defined as:
$\text{PSNR}:=10\log_{10}(\frac{255^{2}\times J_{1}J_{2}}{\|J^{r}-J^{c}\|^{2}})$.

\textbf{2.} SSIM is defined as:
$\text{SSIM}:=\dfrac{(2\mu_{J^{c}}\mu_{J^{r}}+c_{1})(2\sigma_{J^{c}J^{r}}+c_{2})}{(\mu^{2}_{J^{c}}+\mu^{2}_{J^{ref}}+c_{1})(\sigma_{J^{c}}^{2}+\sigma_{J^{r}}^{2}+c_{2})},$
where $\mu_{J^{c}}$ is average of $J^{c}$; $\sigma_{J^{c}}$ is the variance of $J^{c}$, and $\sigma_{J^{c}J^{r}}$ is the covariance of $J^{c}$ and $J^{r}$.

\textbf{3.} FSIM is defined as:
$\text{FSIM}:=\dfrac{\sum_{x\in\Omega S_{L}(x)\cdot PC_{m}(x)}}{\sum_{x\in\Omega PC_{m}(x)}},$
where $S_{L}(x)$ is derived from the phase congruency and the image gradient magnitude of
$J^{c}$ and $J^{r}$ ; $PC_{m}(x)$ is the maximum phase congruency of $PC_{c}$ (for $J^{c}$) and $PC_{r}$ (for
$J^{r}$ ); and $\Omega$ represents the entire airspace of the image. 

When calculating the PSNR, SSIM, and FSIM values for the tensor image, each slice of the tensor is computed first, and then the average value is taken across slices.

\textbf{4.} ERGAS for $\mathcal{A},\mathcal{B}\in \mathbb{R}^{\mathit{J}_{1}\times\mathit{J}_{2}\times\mathit{J}_{3}}$ is defined as:
\begin{equation*}
	\text{ERGAS}:=100\sqrt{\dfrac{\sum_{j=1}^{\mathit{J}_{3}}(\text{\tt mse}(\mathcal{A}(:,:,j)-\mathcal{B}(:,:,j))/(\text{\tt mean2}(\mathcal{A}(:,:,j))))}{\mathit{J}_{3}}},
\end{equation*}
where $\text{\tt mse}$ is mean squared error performance function, and $\text{\tt mean2}$ is average of matrix elements.

\textbf{5.} For a tensor $\mathcal{B}\in \mathbb{R}^{\mathit{J}_{1}\times\mathit{J}_{2}\times\cdots
	\times\mathit{J}_{N}}$, its sampling rate (SR) is defined as $\frac{\text{ num}(\Omega)}{\prod_{j=1}^{N}\mathit{J}_{j}}\times100\%$, where $\Omega$ denotes an index set of observed data; $\text{ num}(\Omega)$ is the total number of elements in index set $\Omega$.
\subsection{Low-rank tensor completion}
In this section, four sets of tensor data are tested: magnetic resonance imaging (MRI), multispectral image (MSI), color video (CV), and Light Field Images (LFI). Among them, MRI and MSI are third-order tensors, CV is a fourth-order tensor, and LFI is a fifth-order tensor. The data sampling method used is purely random sampling. The following LRTC methods are compared: HaLRTC \cite{6138863} and TMCP \cite{chen2024spatiotemporal}, which represent the state-of-the-art approach based on Tucker decomposition. And the t-SVD-based methods: TNN \cite{7782758}, PSTNN \cite{1122020112680}, DCTTNN \cite{Lu_2019_CVPR}, TTNN \cite{ng2020patched}, FTNN \cite{9115254}, WSTNN \cite{2020170}, t-$\epsilon$-LogDet \cite{9872134}, and BEMCP \cite{ZHANG2024110253}. Each method is tested using its respective optimal parameters. Since the t-SVD-based methods are applicable only to third-order tensors, for the tests involving fourth-order and fifth-order tensors, the data are reshaped into third-order tensors before applying these methods.
\subsubsection{MRI completion}
The performance of the proposed method, as well as the comparative methods, is evaluated on MRI\footnote{http://brainweb.bic.mni.mcgill.ca/brainweb/selection\_normal.html} data with dimensions $181\times217\times181$. Initially, the visual quality of the MRI data recovered at sampling rates of 5\%, 10\% and 20\% are demonstrated in Figs. \ref{MRIF1}-\ref{MRIF3}. It is evident that the proposed method outperforms the comparative methods. As shown in Fig. \ref{MRIF1}, the image recovered using the proposed method at a 5\% sampling rate exhibits successful overall reconstruction. 
With an increase in the sampling rate, further image details are progressively restored, as illustrated in Figs. \ref{MRIF2} and \ref{MRIF3}.

\begin{figure}[!tbh] 
	\centering  
	\vspace{0cm} 
	\includegraphics[width=0.9\linewidth]{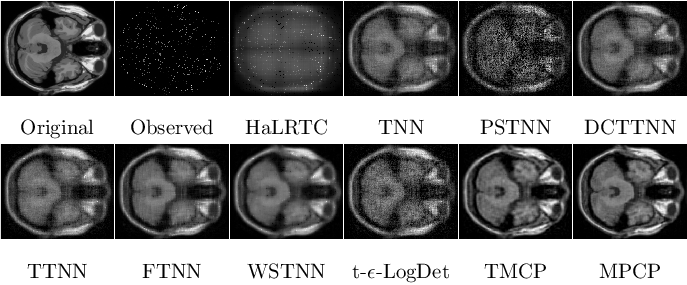} 
	\caption{Visual results of the 40th slice of the MRI data at a sampling rate of 5\%.}
	\label{MRIF1}
\end{figure}
\begin{figure}[!tbh] 
	\centering  
	\vspace{0cm} 
	\includegraphics[width=0.9\linewidth]{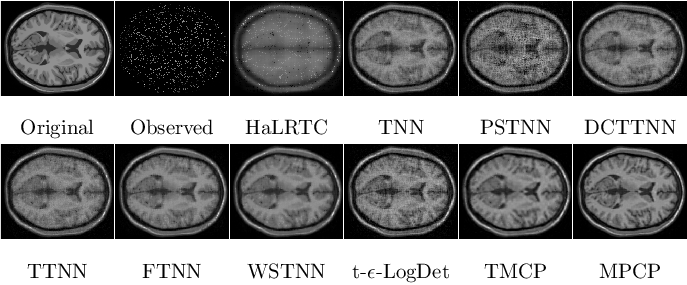} 
	\caption{Visual results of the 80th slice of the MRI data at a sampling rate of 10\%.}
	\label{MRIF2}
\end{figure}
\begin{figure}[!tbh] 
	\centering  
	\vspace{0cm} 
	\includegraphics[width=0.9\linewidth]{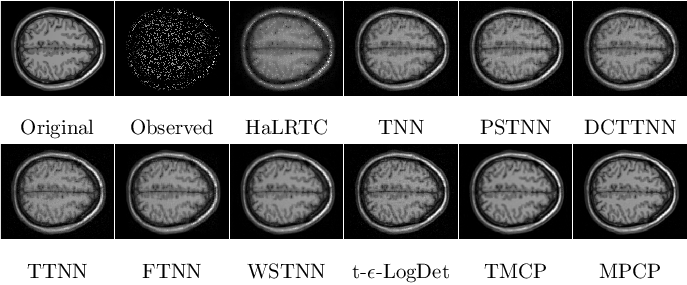} 
	\caption{Visual results of the 120th slice of the MRI data at a sampling rate of 20\%.}
	\label{MRIF3}
\end{figure}

Subsequently, the average quantitative results of the frontal slices of the MRI, restored by all methods at various sampling rates, are presented in Table \ref{table1}. The optimal results are highlighted in bold. The proposed method demonstrates a 2.7 dB higher PSNR value compared to the suboptimal LPRN method at a 5\% sampling rate, alongside significant improvements in SSIM, FSIM, and ERGAS values.
%
\begin{table}[!tbh]
	\centering
	\caption{The PSNR, SSIM, FSIM and ERGAS values for MRI tested by observed and the ten utilized LRTC methods.}\label{table1}
	\medskip\small\renewcommand{\arraystretch}{1.15}
	\resizebox{\textwidth}{!}{%
		\begin{tabular}{||c|cccc|cccc|cccc||}
			\hline
			SR         & \multicolumn{4}{c|}{5\%}                                             & \multicolumn{4}{c|}{10\%}                                           & \multicolumn{4}{c||}{20\%}                                           \\ \hline
			Method     & PSNR            & SSIM           & FSIM           & ERGAS            & PSNR            & SSIM           & FSIM           & ERGAS           & PSNR            & SSIM           & FSIM           & ERGAS           \\ \hline
			Observed   & 11.399          & 0.310          & 0.530          & 1021.079         & 11.635          & 0.323          & 0.565          & 993.760         & 12.148          & 0.350          & 0.612          & 936.741         \\
			HaLRTC     & 17.297          & 0.298          & 0.637          & 537.432          & 20.146          & 0.439          & 0.726          & 389.074         & 24.454          & 0.660          & 0.829          & 234.870         \\
			TNN        & 22.691          & 0.470          & 0.743          & 304.114          & 26.097          & 0.643          & 0.812          & 204.863         & 29.963          & 0.799          & 0.882          & 130.769         \\
			PSTNN      & 16.166          & 0.197          & 0.588          & 608.853          & 22.453          & 0.439          & 0.723          & 307.529         & 29.566          & 0.767          & 0.870          & 137.477         \\
			DCTTNN     & 23.280          & 0.507          & 0.755          & 277.284          & 26.066          & 0.655          & 0.815          & 198.957         & 30.167          & 0.818          & 0.888          & 122.854         \\
			TTNN       & 23.363          & 0.494          & 0.755          & 275.704          & 26.384          & 0.651          & 0.818          & 189.680         & 30.440          & 0.814          & 0.890          & 117.127         \\
			FTNN       & 24.575          & 0.688          & 0.832          & 240.235          & 27.658          & 0.806          & 0.885          & 164.653         & 31.833          & 0.908          & 0.938          & 99.941          \\
			WSTNN      & 25.519          & 0.708          & 0.824          & 212.073          & 29.036          & 0.836          & 0.887          & 139.526         & 33.441          & 0.929          & 0.940          & 83.149          \\
			t-$\epsilon$-LogDet & 21.161          & 0.340          & 0.677          & 357.989          & 26.670          & 0.613          & 0.798          & 192.217         & 31.480          & 0.811          & 0.888          & 109.610         \\
			TMCP       & 27.862          & 0.735          & 0.839          & 154.266          & 30.595          & 0.861          & 0.891          & 112.631         & 34.360          & 0.944          & 0.944          & 73.017          \\
			MPCP       & \textbf{30.573} & \textbf{0.855} & \textbf{0.895} & \textbf{115.228} & \textbf{33.700} & \textbf{0.920} & \textbf{0.932} & \textbf{80.246} & \textbf{37.339} & \textbf{0.962} & \textbf{0.963} & \textbf{52.177} \\ \hline
		\end{tabular}%
	}
\end{table}

\subsubsection{MSI completion}
\begin{figure}[!tbh] 
	\centering  
	\vspace{0cm} 
	\includegraphics[width=0.9\linewidth]{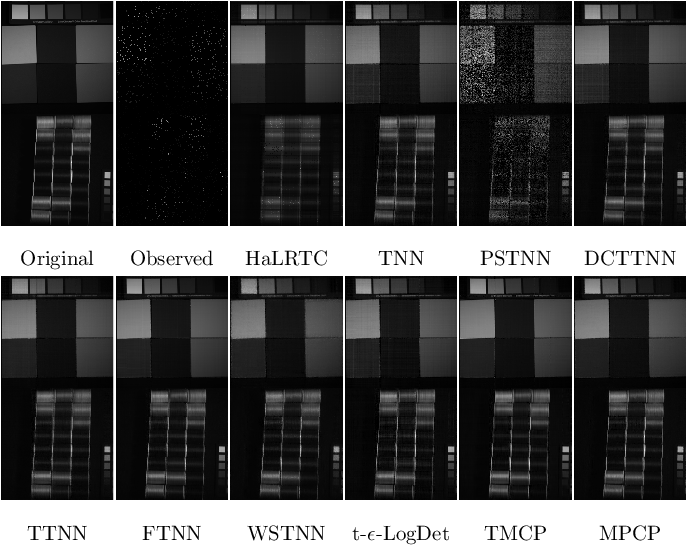} 
	\caption{Visual results for MSI data. The rows of MSIs are in order: sponges, thread\_spools. SR: 5\%. The corresponding bands in each row are: 10, 5.}
	\label{MSIF1}
\end{figure}
The test dataset used in this experiment consists of six MSIs: chart\_and\_stuffed\_toy, thread\_spools, fake\_and\_real\_lemons, fake\_and\_real\_tomatoes, real\_and\_fake\_peppers, and sponges. All testing data have dimensions of $256\times256\times31$, where the spatial resolution is $256\times256$ and the spectral resolution is $31$. Visual results for different sampling rates and bands are presented in Figs. \ref{MSIF1}-\ref{MSIF3}. The specific MSI names and their corresponding bands are provided in the figure captions. 

As shown in Figs. \ref{MSIF1}-\ref{MSIF3}, the visual quality of the results obtained using the MPCP method surpasses that of the comparative methods at all three sampling rates. In particular, Fig. \ref{MSIF1} demonstrates that the MPCP method is more effective in tensor recovery when compared to the suboptimal methods. For example, the image recovered using the MPCP method reveals the sponges more clearly, while the image recovered using the suboptimal FTNN method still contains significant noise and artifacts.

To further emphasize the superiority of the proposed method, the average quantitative results for the six MSIs are summarized in Table \ref{table2}. The proposed method demonstrates a significant improvement, achieving at least 4.2 dB higher PSNR for image recovery compared to the suboptimal WSTNN method at both 10\% and 20\% sampling rates. Moreover, at the same sampling rates, the SSIM, FSIM, and ERGAS values obtained by the proposed method are markedly superior to those of the WSTNN method. At a 5\% sampling rate, where the suboptimal method is the TMCP method, the proposed method still outperforms TMCP method by 3.5 dB in terms of PSNR and also achieves a substantially higher SSIM value.

\begin{figure}[!tbh] 
	\centering  
	\vspace{0cm} 
	\includegraphics[width=0.9\linewidth]{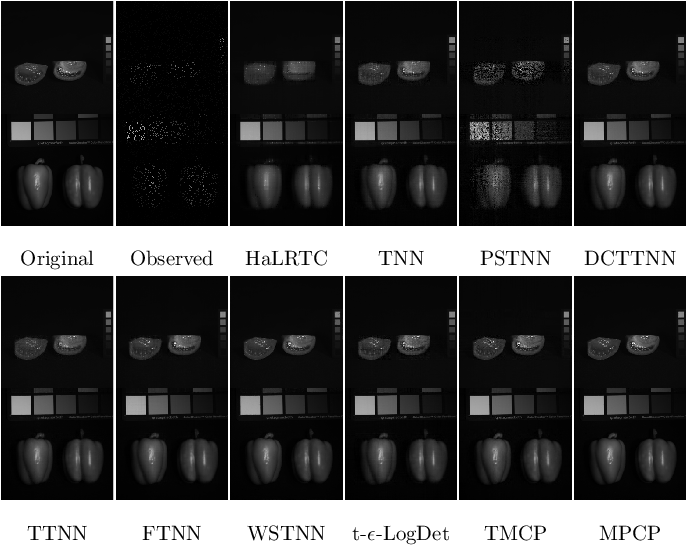} 
	\caption{Visual results for MSI data. The rows of MSIs are in order: fake\_and\_real\_tomatoes, real\_and\_fake\_peppers. SR: 10\%. The corresponding bands in each row are: 20, 15.}
	\label{MSIF2}
\end{figure}
\begin{figure}[!tbh] 
	\centering  
	\vspace{0cm} 
	\includegraphics[width=0.9\linewidth]{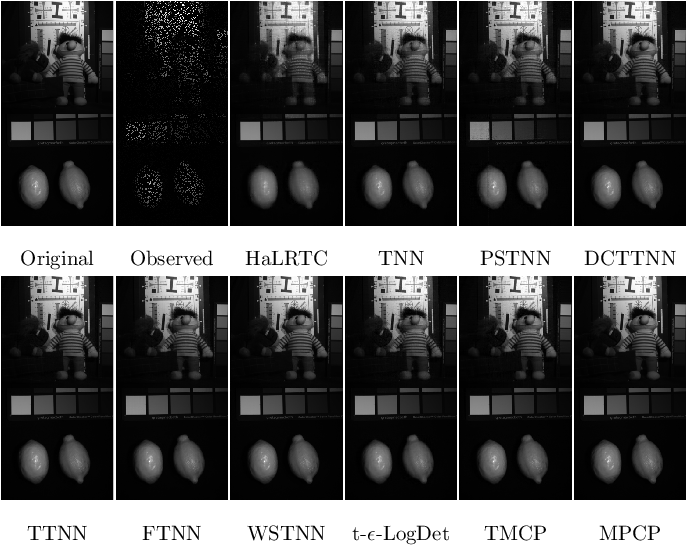} 
	\caption{Visual results for MSI data. The rows of MSIs are in order: chart\_and\_stuffed\_toy, fake\_and\_real\_lemons. SR: 20\%. The corresponding bands in each row are: 30, 25.}
	\label{MSIF3}
\end{figure}
\begin{table}[!tbh]
	\centering
	\caption{The average PSNR, SSIM, FSIM and ERGAS values for six MSIs tested by observed and the ten utilized LRTC methods.}\label{table2}
	\medskip\small\renewcommand{\arraystretch}{1.15}
	\resizebox{\textwidth}{!}{%
		\begin{tabular}{||c|cccc|cccc|cccc||}
			\hline
			SR         & \multicolumn{4}{c|}{5\%}                                            & \multicolumn{4}{c|}{10\%}                                           & \multicolumn{4}{c||}{20\%}                                           \\ \hline
			Method     & PSNR            & SSIM           & FSIM           & ERGAS           & PSNR            & SSIM           & FSIM           & ERGAS           & PSNR            & SSIM           & FSIM           & ERGAS           \\ \hline
			Observed   & 14.956          & 0.167          & 0.675          & 891.700         & 15.190          & 0.210          & 0.665          & 867.956         & 15.703          & 0.287          & 0.654          & 818.238         \\
			HaLRTC     & 28.138          & 0.867          & 0.893          & 212.430         & 33.040          & 0.931          & 0.939          & 123.174         & 38.396          & 0.971          & 0.972          & 68.326          \\
			TNN        & 32.694          & 0.897          & 0.915          & 127.707         & 37.576          & 0.954          & 0.958          & 75.302          & 43.342          & 0.983          & 0.984          & 40.452          \\
			PSTNN      & 18.744          & 0.541          & 0.660          & 565.723         & 24.036          & 0.737          & 0.804          & 314.595         & 39.247          & 0.967          & 0.968          & 67.796          \\
			DCTTNN     & 35.034          & 0.935          & 0.942          & 95.770          & 39.971          & 0.973          & 0.974          & 56.226          & 45.945          & 0.991          & 0.991          & 29.589          \\
			TTNN       & 34.907          & 0.940          & 0.946          & 97.114          & 40.447          & 0.979          & 0.980          & 51.580          & 47.066          & 0.994          & 0.994          & 24.567          \\
			FTNN       & 36.514          & 0.956          & 0.958          & 81.911          & 41.244          & 0.981          & 0.980          & 48.205          & 46.930          & 0.993          & 0.993          & 26.680          \\
			WSTNN      & 33.452          & 0.822          & 0.929          & 166.546         & 42.221          & 0.988          & 0.987          & 38.839          & 49.899          & 0.997          & 0.997          & 16.421          \\
			t-$\epsilon$-LogDet & 33.826          & 0.890          & 0.909          & 113.213         & 39.689          & 0.963          & 0.964          & 59.623          & 45.345          & 0.987          & 0.987          & 33.042          \\
			TMCP       & 37.126          & 0.949          & 0.950          & 73.914          & 42.180          & 0.978          & 0.976          & 42.218          & 47.418          & 0.991          & 0.991          & 23.654          \\
			MPCP       & \textbf{40.676} & \textbf{0.981} & \textbf{0.978} & \textbf{46.355} & \textbf{46.646} & \textbf{0.995} & \textbf{0.994} & \textbf{23.322} & \textbf{54.117} & \textbf{0.999} & \textbf{0.999} & \textbf{10.135} \\ \hline
		\end{tabular}%
	}
\end{table}

\subsubsection{CV completion}
We conduct tests on three CVs\footnote{http://trace.eas.asu.edu/yuv/}, namely akiyo, coastguard, and hall, each with dimensions 
$144 \times 176 \times 3 \times 50$, where the number of frames is $50$ and each frame is a color image of size $144 \times 176 \times 3$. The visual results for these three CVs are presented in Figs. \ref{CVF1}-\ref{CVF3}. The number of frames and corresponding sampling rates for each CV are specified in the figure captions. 

As shown in Figs. \ref{CVF1}-\ref{CVF3}, the proposed method consistently produces superior visual recovery compared to the other methods. In particular, Fig. \ref{CVF1} illustrates that the proposed method effectively restores detailed features of the person in the image, such as facial characteristics, while the comparison methods still exhibit considerable blurriness in recovering these details.

To further substantiate the superiority of the proposed method, the average quantitative results for the three CVs are summarized in Table \ref{table3}. Evaluating metrics such as PSNR, SSIM, FSIM, and ERGAS values, the proposed method outperforms the other methods at various sampling rates. Notably, it achieves a PSNR value that is at least 2.3 dB higher than the suboptimal WSTNN method.

\begin{figure}[!tbh] 
	\centering  
	\vspace{0cm} 
	\includegraphics[width=0.9\linewidth]{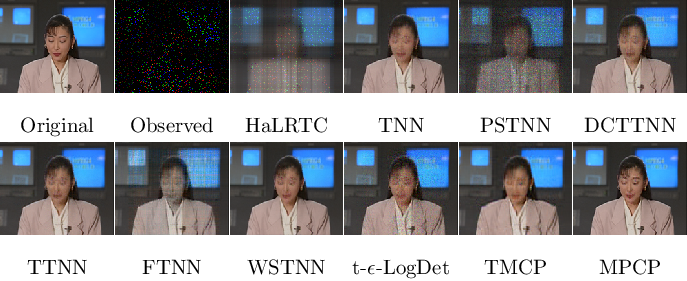} 
	\caption{Visual results of the 15th frame of the akiyo data at a sampling rate of 5\%.}
	\label{CVF1}
\end{figure}
\begin{figure}[!tbh] 
	\centering  
	\vspace{0cm} 
	\includegraphics[width=0.9\linewidth]{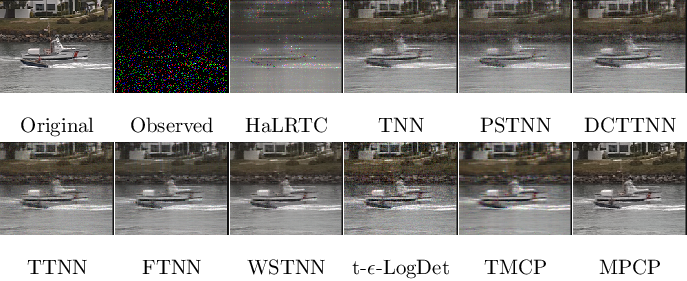} 
	\caption{Visual results of the 30th frame of the coastguard data at a sampling rate of 10\%.}
	\label{CVF2}
\end{figure}
\begin{figure}[!tbh] 
	\centering  
	\vspace{0cm} 
	\includegraphics[width=0.9\linewidth]{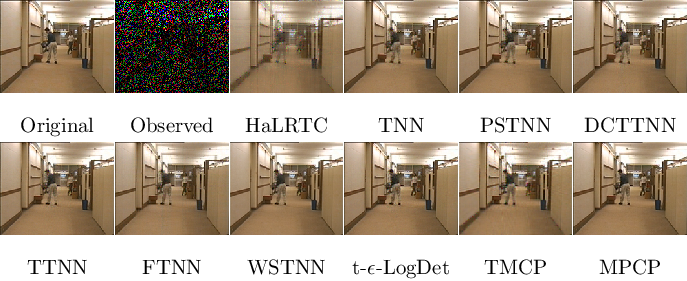} 
	\caption{Visual results of the 45th frame of the hall data at a sampling rate of 20\%.}
	\label{CVF3}
\end{figure}
\begin{table}[!tbh]
	\centering
	\caption{The average PSNR, SSIM, FSIM and ERGAS values for three CVs tested by observed and the ten utilized LRTC methods.}
	\label{table3}
	\medskip\small\renewcommand{\arraystretch}{1.15}
	\resizebox{\textwidth}{!}{%
		\begin{tabular}{||c|cccc|cccc|cccc||}
			\hline
			SR         & \multicolumn{4}{c|}{5\%}                                            & \multicolumn{4}{c|}{10\%}                                           & \multicolumn{4}{c||}{20\%}                                           \\ \hline
			Method     & PSNR            & SSIM           & FSIM           & ERGAS           & PSNR            & SSIM           & FSIM           & ERGAS           & PSNR            & SSIM           & FSIM           & ERGAS           \\ \hline
			Observed   & 6.554           & 0.012          & 0.417          & 1144.119        & 6.789           & 0.021          & 0.433          & 1113.472        & 7.301           & 0.038          & 0.456          & 1049.787        \\
			HaLRTC     & 17.036          & 0.463          & 0.685          & 345.162         & 20.935          & 0.611          & 0.769          & 221.104         & 25.049          & 0.773          & 0.862          & 141.290         \\
			TNN        & 27.879          & 0.791          & 0.898          & 111.919         & 30.760          & 0.855          & 0.931          & 84.050          & 33.835          & 0.907          & 0.955          & 61.992          \\
			PSTNN      & 15.637          & 0.326          & 0.696          & 403.452         & 27.398          & 0.796          & 0.896          & 113.445         & 33.383          & 0.906          & 0.953          & 64.321          \\
			DCTTNN     & 27.375          & 0.774          & 0.893          & 117.217         & 29.970          & 0.842          & 0.924          & 89.617          & 33.334          & 0.904          & 0.954          & 63.695          \\
			TTNN       & 28.413          & 0.805          & 0.904          & 106.046         & 31.235          & 0.866          & 0.935          & 80.345          & 34.775          & 0.923          & 0.962          & 55.778          \\
			FTNN       & 25.037          & 0.756          & 0.866          & 148.398         & 28.449          & 0.852          & 0.915          & 99.763          & 32.132          & 0.922          & 0.953          & 66.211          \\
			WSTNN      & 29.485          & 0.864          & 0.917          & 92.928          & 32.722          & 0.917          & 0.949          & 67.230          & 36.500          & 0.958          & 0.973          & 45.013          \\
			t-$\epsilon$-LogDet & 20.854          & 0.494          & 0.783          & 244.032         & 31.539          & 0.846          & 0.935          & 81.087          & 34.895          & 0.905          & 0.959          & 57.988          \\
			TMCP       & 27.286          & 0.788          & 0.879          & 112.108         & 28.845          & 0.833          & 0.903          & 94.287          & 31.138          & 0.892          & 0.935          & 73.240          \\
			MPCP       & \textbf{31.786} & \textbf{0.888} & \textbf{0.944} & \textbf{75.478} & \textbf{35.278} & \textbf{0.936} & \textbf{0.966} & \textbf{52.651} & \textbf{39.432} & \textbf{0.970} & \textbf{0.984} & \textbf{33.166} \\ \hline
		\end{tabular}%
	}
\end{table}

\subsubsection{LFI completion}
In this section, the proposed method is evaluated using the Antinous dataset from the HCI 4D LFI\footnote{https://lightfield-analysis.uni-konstanz.de/}. The size of the Antinous data is $128\times128\times3\times9\times9$, making it a fifth-order tensor. This dataset consists of $81$ different views arranged in a $9\times9$ grid. To quantitatively assess the performance of the completion method, we computed the average quantitative results for the 81 color images, each with dimensions $128\times128\times3$. The average quantitative results for the Antinous dataset are presented in Table \ref{table4}. 

Considering metrics such as PSNR, SSIM, FSIM, and ERGAS, the proposed method outperforms all other methods across various sampling rates. The suboptimal method in this comparison is the WSTNN method. Specifically, the proposed method achieves a PSNR value that is at least 4.0 dB higher than that of the suboptimal WSTNN method.

Additionally, the visual results at 5\% - 20\% sampling rates are shown in Figs. \ref{LFIF1}-\ref{LFIF3}. As illustrated, the proposed method yields the best visual recovery. Notably, the reconstruction of the statue's head is significantly clearer and more detailed compared to the other methods.
\begin{table}[!tbh]
	\centering
	\caption{The average PSNR, SSIM, FSIM and ERGAS values for LFI tested by observed and the ten utilized LRTC methods.}
	\label{table4}
	\medskip\small\renewcommand{\arraystretch}{1.15}
	\resizebox{\textwidth}{!}{%
		\begin{tabular}{||c|cccc|cccc|cccc||}
			\hline
			SR         & \multicolumn{4}{c|}{5\%}                                            & \multicolumn{4}{c|}{10\%}                                           & \multicolumn{4}{c||}{20\%}                                           \\ \hline
			Method     & PSNR            & SSIM           & FSIM           & ERGAS           & PSNR            & SSIM           & FSIM           & ERGAS           & PSNR            & SSIM           & FSIM           & ERGAS           \\ \hline
			Observed   & 8.107           & 0.014          & 0.510          & 1035.474        & 8.341           & 0.021          & 0.462          & 1007.901        & 8.853           & 0.030          & 0.402          & 950.167         \\
			HaLRTC     & 17.753          & 0.593          & 0.778          & 341.075         & 22.228          & 0.707          & 0.839          & 203.733         & 28.175          & 0.861          & 0.917          & 102.896         \\
			TNN        & 31.916          & 0.912          & 0.949          & 69.625          & 34.771          & 0.947          & 0.969          & 51.419          & 38.661          & 0.973          & 0.984          & 33.961          \\
			PSTNN      & 18.718          & 0.467          & 0.743          & 307.351         & 31.556          & 0.910          & 0.948          & 71.712          & 37.947          & 0.970          & 0.982          & 36.886          \\
			DCTTNN     & 31.760          & 0.907          & 0.949          & 70.430          & 34.041          & 0.939          & 0.966          & 55.074          & 37.592          & 0.968          & 0.982          & 37.638          \\
			TTNN       & 32.758          & 0.929          & 0.959          & 62.807          & 36.189          & 0.963          & 0.979          & 43.003          & 40.609          & 0.984          & 0.991          & 26.179          \\
			FTNN       & 28.609          & 0.847          & 0.915          & 98.810          & 31.028          & 0.903          & 0.946          & 75.189          & 34.058          & 0.945          & 0.968          & 53.542          \\
			WSTNN      & 33.617          & 0.950          & 0.963          & 58.507          & 37.481          & 0.975          & 0.981          & 38.595          & 42.447          & 0.989          & 0.993          & 22.472          \\
			t-$\epsilon$-LogDet & 24.161          & 0.720          & 0.865          & 163.363         & 36.073          & 0.952          & 0.972          & 44.400          & 40.512          & 0.979          & 0.987          & 27.869          \\
			TMCP       & 31.909          & 0.912          & 0.940          & 67.109          & 33.974          & 0.937          & 0.956          & 52.922          & 36.086          & 0.959          & 0.971          & 41.543          \\
			MPCP       & \textbf{37.658} & \textbf{0.971} & \textbf{0.981} & \textbf{39.040} & \textbf{42.075} & \textbf{0.986} & \textbf{0.991} & \textbf{23.979} & \textbf{46.931} & \textbf{0.994} & \textbf{0.996} & \textbf{13.903} \\ \hline
		\end{tabular}%
	}
\end{table}
\begin{figure}[!tbh] 
	\centering  
	\vspace{0cm} 
	\includegraphics[width=0.9\linewidth]{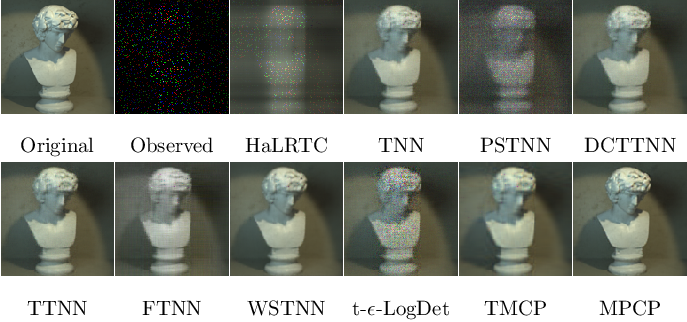} 
	\caption{Visual results of LFI data at a sampling rate of 5\%. The image located at the (3, 3)-th grid of Antinous}
	\label{LFIF1}
\end{figure}
\begin{figure}[!tbh] 
	\centering  
	\vspace{0cm} 
	\includegraphics[width=0.9\linewidth]{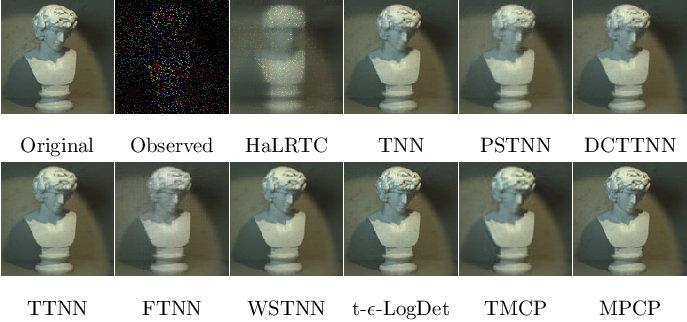} 
	\caption{Visual results for MRI data. MR: top row is 95\%, and last  row is 90\%. The corresponding slices in each row are: 50, 100.}
	\label{LFIF2}
\end{figure}
\begin{figure}[!tbh] 
	\centering  
	\vspace{0cm} 
	\includegraphics[width=0.9\linewidth]{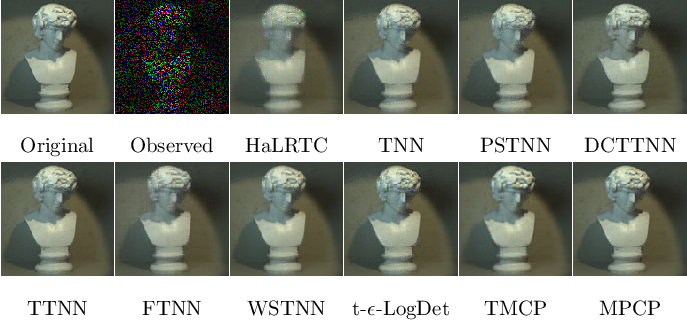} 
	\caption{Visual results for MRI data. MR: top row is 95\%, and last  row is 90\%. The corresponding slices in each row are: 50, 100.}
	\label{LFIF3}
\end{figure}

\subsection{Discussion}
\subsubsection{Parameters setting}
The proposed method utilizes the N-tubal rank framework, which allows for the selection of the parameters $\beta$ and $\rho$ based on the parameter settings from \cite{2020170}. In this section, we analyze the impact of different parameters on the tensor $p$-th order $\tau$ norm. All tests are conducted using the MSI dataset `` chart\_and\_stuffed\_toy ''.

Given that the choice of parameter $\tau$ depends on the value of $p$, the performance of the proposed method is evaluated at a 5\% sampling rate under various settings of $p$ and $\tau^p$. The parameter $p$ is selected from the set $\{0.1: 0.1: 0.9\}$, and $\tau^p$ is chosen from the set $\{50, 100, 150, 200\}$. The PSNR values corresponding to different conditions are presented in Table \ref{TPa1}.
\begin{table}[!tbh]
	\caption{The PSNR value with different $\tau^p$ and $p$ of tensor $p$-th order $\tau$ norm.}
	\label{TPa1}
	\centering
	\medskip\small\renewcommand{\arraystretch}{1.15}
	\resizebox{\textwidth}{!}{%
		\begin{tabular}{||c|ccccccccc||}
			\hline
			\diagbox{$\tau^p$}{$p$}   & 0.9             & 0.8             & 0.7             & 0.6             & 0.5             & 0.4             & 0.3             & 0.2             & 0.1             \\ \hline
			50  & 16.899          & 16.901          & 16.903          & 16.907          & 16.911          & 16.914          & 16.918          & 16.924          & 16.942          \\
			100 & \textbf{34.734} & \textbf{34.824} & \textbf{34.923} & \textbf{35.031} & \textbf{35.143} & \textbf{35.258} & \textbf{35.362} & \textbf{35.441} & 35.484          \\
			150 & 34.371          & 34.473          & 34.590          & 34.724          & 34.869          & 35.037          & 35.204          & 35.359          & \textbf{35.508} \\
			200 & 34.086          & 34.202          & 34.332          & 34.479          & 34.639          & 34.812          & 34.993          & 35.171          & 35.338          \\ \hline
		\end{tabular}%
	}
\end{table}

As shown in the table, recovery performance improves significantly as the value of $p$ decreases. Specifically, when $p=0.1$, the PSNR value increases by 0.75 dB compared to when $p=0.9$. This finding supports our theoretical analysis: for the same $\tau^p$, the penalty applied to smaller singular values increases as $p$ decreases. Based on this observation, the value of $p$ is selected to be 0.1. Additionally, Table \ref{TPa2} lists the optimal $\tau^p$ values for different images across various sampling rates.

\begin{table}[!tbh]
	\centering
	\caption{The optimal $\tau^p$ values for different images at various sampling rates.}
	\label{TPa2}
	\medskip\small\renewcommand{\arraystretch}{1.15}
		\begin{tabular}{||c|ccc|c|ccc||}
			\hline
			\diagbox{Image}{SR}                        & 5\% & 10\% & 20\% & \diagbox{Image}{SR}       & 5\% & 10\% & 20\% \\ \hline
			chart\_and\_stuffed\_toy  & 120 & 70   & 40   & MRI        & 80  & 50   & 30   \\
			fake\_and\_real\_lemons   & 50  & 30   & 20   & akiyo      & 120 & 50   & 50   \\
			fake\_and\_real\_tomatoes & 50  & 30   & 20   & coastguard & 330 & 180  & 90   \\
			real\_and\_fake\_peppers  & 60  & 40   & 30   & hall       & 290 & 100  & 60   \\
			sponges                   & 110 & 80   & 60   & antinous   & 140 & 90   & 80   \\
			thread\_spools            & 70  & 40   & 30   &            &     &      &      \\ \hline
		\end{tabular}
\end{table}

\subsubsection{Ablation study}

In this section, the advantages of the MPCP function over the traditional MCP function are further validated. A comparison is made with a variant model that uses the MCP function (referred to as the MCP method). Quantitative results for various images and sampling rates are presented in Table \ref{TAbs}. The results demonstrate that the MPCP method outperforms the MCP method. Moreover, at higher sampling rates, the performance gap between the MPCP and MCP methods continues to expand, further supporting our theoretical analysis. Specifically, the MPCP function not only protects larger singular values more effectively but also penalizes smaller singular values more efficiently, thereby addressing the MCP function's limitation in penalizing small singular values.

\begin{table}[!tbh]
	\centering
	\caption{The quantitative results at different images and sampling rates}
	\label{TAbs}
	\medskip\small\renewcommand{\arraystretch}{1.15}
	\resizebox{\textwidth}{!}{%
		\begin{tabular}{||c|c|cc|cc|cc|cc||}
			\hline
			& Image  & \multicolumn{2}{c|}{chart\_and\_stuffed\_toy} & \multicolumn{2}{c|}{MRI}   & \multicolumn{2}{c|}{akiyo} & \multicolumn{2}{c||}{antinous} \\ \hline
			SR                  & Method & MCP               & MPCP                      & MCP     & MPCP             & MCP     & MPCP             & MCP       & MPCP              \\ \hline
			\multirow{4}{*}{5\%}  & PSNR   & 34.555            & \textbf{35.553}           & 29.112  & \textbf{30.573}  & 36.041  & \textbf{36.849}  & 35.545    & \textbf{37.658}   \\
			& SSIM   & 0.950             & \textbf{0.960}            & 0.823   & \textbf{0.855}   & 0.974   & \textbf{0.977}   & 0.962     & \textbf{0.971}    \\
			& FSIM   & 0.958             & \textbf{0.964}            & 0.875   & \textbf{0.895}   & 0.983   & \textbf{0.986}   & 0.973     & \textbf{0.981}    \\
			& ERGAS  & 75.911            & \textbf{67.630}           & 137.366 & \textbf{115.228} & 41.464  & \textbf{38.048}  & 48.398    & \textbf{39.040}   \\ \hline
			\multirow{4}{*}{10\%} & PSNR   & 40.445            & \textbf{42.215}           & 32.058  & \textbf{33.700}  & 40.228  & \textbf{41.689}  & 40.279    & \textbf{42.075}   \\
			& SSIM   & 0.984             & \textbf{0.989}            & 0.898   & \textbf{0.920}   & 0.989   & \textbf{0.991}   & 0.983     & \textbf{0.986}    \\
			& FSIM   & 0.987             & \textbf{0.990}            & 0.917   & \textbf{0.932}   & 0.992   & \textbf{0.994}   & 0.988     & \textbf{0.991}    \\
			& ERGAS  & 38.898            & \textbf{31.639}           & 97.716  & \textbf{80.246}  & 26.498  & \textbf{22.814}  & 28.874    & \textbf{23.979}   \\ \hline
			\multirow{4}{*}{20\%} & PSNR   & 48.004            & \textbf{50.796}           & 35.561  & \textbf{37.339}  & 44.601  & \textbf{46.596}  & 44.989    & \textbf{46.931}   \\
			& SSIM   & 0.996             & \textbf{0.998}            & 0.950   & \textbf{0.962}   & 0.995   & \textbf{0.996}   & 0.992     & \textbf{0.994}    \\
			& FSIM   & 0.997             & \textbf{0.998}            & 0.953   & \textbf{0.963}   & 0.997   & \textbf{0.998}   & 0.995     & \textbf{0.996}    \\
			& ERGAS  & 17.120            & \textbf{12.361}           & 64.788  & \textbf{52.177}  & 16.582  & \textbf{13.501}  & 17.153    & \textbf{13.903}   \\ \hline
		\end{tabular}%
	}
\end{table}
\section{Conclusion}

This paper proposes a novel non-convex function, the MPCP function. The MPCP function not only effectively protects large singular values but also imposes a strong penalty on small singular values, addressing a key limitation of the MCP function in penalizing smaller singular values. Both theoretical analysis and experimental results demonstrate the superiority of the MPCP function over the MCP function. As a non-convex relaxation for the LRTC problem, the tensor $p$-th order $\tau$ norm is derived. We then investigate the LRTC model based on the MPCP function and its corresponding solution algorithm. Extensive experiments reveal that the proposed method outperforms the comparison methods in both visual and numerical quantitative results. In the future, we aim to explore the application of the proposed function in tensor robust principal component analysis (TRPCA), specifically examining how its properties may change depending on the problem at hand. Additionally, given the use of the $l_1$ norm in TRPCA, we intend to investigate whether the proposed function can serve as a non-convex relaxation and analyze its associated properties.

\section*{Acknowledgments}

This work was supported by the National Nature Science Foundation of China under Grant 12471353.
\end{document}